\documentclass[11pt]{article}
\usepackage{amsthm,fullpage,amssymb}
\usepackage[cmex10]{amsmath}
\usepackage{amsthm,amssymb,url,graphicx}
\newtheorem*{theorem*}{Theorem}
\newtheorem{lemma}{Lemma}
\newtheorem{proposition}{Proposition}
\newtheorem{corollary}{Corollary}
\newtheorem{theorem}{Theorem}
\newtheorem{definition}{Definition}
\newtheorem{example}{Example}
\newtheorem{claim}{Claim}
\newtheorem{fact}{Fact}
\let\OLDthebibliography\thebibliography
\renewcommand\thebibliography[1]{
  \OLDthebibliography{#1}
  \setlength{\parskip}{.8pt}
  \setlength{\itemsep}{0pt plus 0.3ex}
}

\newcommand{\constB}{\mathcal{B}}
\newcommand{\constC}{\mathcal{C}}
\newcommand{\constD}{\mathcal{D}}

\newcommand{\defeq}{:=}
\newtheorem{Alg}{Algorithm}
\newcommand{\myalg}[4][0cm]{
\medskip
\small{
\fbox{
\parbox{6.0in}{\vspace{#1}
\begin{Alg}\label{#2}{\textsc{ #3}}
\vspace{.1cm}\\ \emph{ #4}
\end{Alg}
}}
\medskip
}}

\newcommand{\E}{\text{E}}
\newcommand{\eps}{\epsilon}
\newcommand{\FF}{\mathcal{F}}

\newcommand{\Eu}[1]{\underset{#1}{\E}}
\newcommand{\R}{\mathbb{R}}
\newcommand{\poi}{\text{poi}}
\newcommand{\ii}{r}
\newcommand{\cc}{0.1}
\author{Gregory Valiant\thanks{This work is supported in part by NSF CAREER Award CCF-1351108.}\\
Stanford University\\
valiant@stanford.edu
\and
Paul Valiant\thanks{This work is supported in part by a Sloan Research Fellowship.}\\
Brown University\\
pvaliant@gmail.com
}
\title{Instance Optimal Learning}

\begin{document}
\maketitle

\begin{abstract}
We consider the following basic learning task:  given independent draws from an unknown distribution over a discrete support, output an approximation of the distribution that is as accurate as possible in $\ell_1$ distance (equivalently, total variation distance, or ``statistical distance'').   Perhaps surprisingly, it is often possible to ``de-noise'' the empirical distribution of the samples to return an approximation of the true distribution that is significantly more accurate than the empirical distribution, \emph{without relying on any prior assumptions on the distribution}.   We present an \emph{instance optimal} learning algorithm which optimally performs this de-noising for every distribution for which such a de-noising is possible.   More formally, given $n$ independent draws from a distribution $p$, our algorithm returns a labelled vector whose expected distance from $p$ is equal to the minimum possible expected error that could be obtained by any algorithm that knows the true unlabeled vector of probabilities of distribution $p$ and simply needs to assign labels, up to an additive subconstant term that is independent of $p$ and goes to zero as $n$ gets large.   This somewhat surprising result has several conceptual implications, including the fact that, for any large sample, Bayesian assumptions on the ``shape'' or bounds on the tail probabilities of a distribution over discrete support are not helpful for the task of learning the distribution.

As a consequence of our techniques, we also show that given a set of $n$ samples from an arbitrary distribution, one can accurately estimate the expected number of distinct elements that will be observed in a sample of any size up to $n \log n$.  This sort of extrapolation is practically relevant, particularly to domains such as genomics where it is important to understand how much more might be discovered given larger sample sizes, and we are optimistic that our approach is practically viable.
\end{abstract}

\thispagestyle{empty}
\newpage
\setcounter{page}{1}
\section{Introduction}
Given independent draws from an unknown distribution over an unknown discrete support, what is the best way to aggregate those samples into an approximation of the true distribution?  This is, perhaps, \emph{the} most fundamental learning problem.   The most obvious and most widely employed approach is to simply output the empirical distribution of the sample.  To what extent can one improve over this naive approach?   To what extent can one ``de-noise'' the empirical distribution, without relying on any assumptions on the structure of the underlying distribution?

Perhaps surprisingly, there are many settings in which de-noising can be done without a priori assumptions on the distribution.  We begin by presenting two motivating examples illustrating rather different settings in which significant de-noising of the empirical distribution is possible.

\begin{example}
Suppose you are given 100,000 independent draws from some unknown distribution, and you find that there are roughly 1,000 distinct elements, each of which appears roughly 100 times.  Furthermore, suppose you compute the variance in the number of times the different domain elements occur, and it is close to 100.  Based on these samples, you can confidently deduce that the true distribution is very close to a uniform distribution over 1,000 domain elements, and that the true probability of a domain element seen 90 times is roughly the same as that of an element observed 110 times.  The basic reasoning is as follows: if the true distribution were the uniform distribution, then the noise from the random sampling would exhibit the observed variance in the number of occurrences;  if there was any significant variation in the true probabilities of the different domain elements, then, combined with the noise added via the random sampling, the observed variance would be noticeably larger than 100.   The $\ell_1$ error of the empirical distribution would be roughly $0.1$, whereas the ``de-noised'' distribution would have error less than $0.01.$
\end{example}

\begin{example}
Suppose you are given 1,000 independent draws from an unknown distribution, and all 1000 samples are unique domain elements.  You can safely conclude that the combined probability of all the observed domain elements is likely to be much less than 1/100---if this were not the case, one would expect at least one of the observed elements to occur twice in the sample.  Hence the empirical distribution of the samples is likely to have $\ell_1$ distance nearly $2$ from the true distribution, whereas this reasoning would suggest that one should ascribe a total probability mass of at most $1/100$ to the observed domain elements.
\end{example}

In both of the above examples, the key to the ``de-noising'' was the realization that the true distributions possessed some structure---structure that was both easily deduced from the samples, and structure that, once known, could then be leveraged to de-noise the empirical distribution.   Our main result is an algorithm which de-noises the empirical distribution as much as is possible, whenever such denoising is possible.  Specifically, our algorithm achieves, up to a subconstant term that vanishes as the sample size increases, the best error that any algorithm could achieve---even an algorithm that is given the unlabeled vector of  true probabilities and simply needs to correctly label the probabilities.

\begin{theorem}\label{thm:main}
There is a function $err(n)$ that goes to zero as $n$ gets large, and an algorithm, which given $n$ independent draws from any distribution $p$ of discrete support, outputs a labelled vector $q$, such that $$\E\left[||p-q||_1\right] \le opt(p,n)+err(n),$$ where $opt(p,n)$ is the minimum expected error that any algorithm could achieve on the following learning task: given $p$, and given $n$ samples drawn independently from a distribution that is identical to $p$ up to an arbitrary relabeling of the domain elements, learn the distribution.
\end{theorem}

The performance guarantees of the above algorithm can be equivalently stated as follows:  let $S \xleftarrow[n]{  } p$ denote that $S$ is a set of $n$ independent draws from distribution $p$, and let $\pi(p)$ denote a distribution that is identical to $p$, up to relabeling the domain elements with arbitrary distinct new labels given by the mapping $\pi$.  Our algorithm, which maps a set of samples $S$ to a labelled vector $q=f(S)$, satisfies the following:  For any distribution $p$,
$$\Eu{S\xleftarrow[n]{  }  p}\left[||p-q||_1\right] \le \min_{\text{algs }\mathcal{A}} \max_{\pi} \left( \Eu{S \xleftarrow[n]{  } \pi(p)}\left[\pi(p) - \mathcal{A}(S)\right]\right) + o_n(1),$$  where $o_n(1) \rightarrow 0$ as $n \rightarrow \infty$ is independent of $p$ and depends only on $n$.

One surprising implication of the above result is that, for large samples, prior knowledge of the ``shape'' of the distribution, or knowledge of the rate of decay of the tails of the distribution, cannot improve the accuracy of the learning task.   For example, typical Bayesian assumptions that the frequency of words in natural language satisfy Zipf distributions, or the frequencies of different species of bacteria in the human gut satisfy Gamma distributions or various power-law distributions, etc, can improve the expected error of the learned distribution by at most a vanishing function of the sample size.
\medskip

The key intuition behind this optimal de-noising, and the core of our algorithm, is the ability to very accurately approximate the \emph{unlabeled} vector of probabilities of the true distribution, given access to independent samples.  In some sense, our result can be interpreted as the following statement: up to an additive subconstant factor, one can always  recover an approximation of the unlabeled vector of probabilities more accurately than one can disambiguate and label such a vector.   That is, if one has enough samples to accurately label the unlabeled vector of probabilities, then one also has more than enough samples to accurately learn that unlabeled vector.   Of course, this statement can only hold up to some additive error term, as the following example illustrates.

\begin{example}
Given samples drawn from a distribution supported on two unknown domain elements, if one is told that both probabilities are exactly $1/2$, then as soon as one observes both domain elements, one knows the distribution exactly, and thus the expected error given $n$ samples will be $O(1/2^n)$ as this bounds the probability that one of the two domain elements is not observed in a set of $n$ samples.  Without the prior knowledge that the two probabilities are $1/2$, the best algorithm will have expected error $\approx 1/\sqrt{n}$.
\end{example}

The above example illustrates that prior knowledge of the vector of probabilities can be helpful.  Our result, however, shows that this phenomena only occurs to a significant extent for very small sample sizes; for larger samples, no distribution  exists for which prior knowledge of the vector of probabilities improves the expected error of a learning algorithm beyond a universal subconstant additive term that goes to zero as a function of the sample size.
\medskip

\subsection{Our Approach}

Our algorithm proceeds via two steps.  In the first step, the samples are used to output an approximation of the vector of true probabilities.   We show that, with high probability over the randomness of the $n$ independent draws from the distribution,  we accurately recover the portion of the vector of true probabilities consisting of values asymptotically larger than $1/(n \log n).$  Note that the empirical distribution accurately estimates probabilities down to $\approx 1/n$---indeed the vector of empirical probabilities are all multiples of $1/n$.  The characterization of the first phase of our algorithm can be interpreted as showing that we recover the vector of  probabilities essentially to the accuracy that the empirical distribution would have if it were based on $n \log n$ samples, rather than only $n$ samples.  Of course, this surprisingly accurate reconstruction comes with the caveat that we are only recovering the unlabeled vector of probabilities---we do not know which domain elements correspond to the various probabilities.

The second step of our algorithm leverages the accurate approximation of the unlabeled vector of probabilities to optimally assign probability values to each of the observed domain elements.  For some intuition into this step, first suppose you know the exact vector of unlabelled probabilities.  Consider the following optimization problem: given $n$ independent draws from distribution $p$, and an unlabeled vector $v$ representing the true vector of probabilities of distribution $p$, for each observed domain element $x$, assign the probability $q(x)$ that minimizes the expected $\ell_1$ distance $|q(x)-p(x)|$.   As the following example illustrates, this problem is more subtle than it might initially seem; intuitive schemes such as assigning the $i$th largest probability in $v$ to the element with the $i$th largest empirical probability is \emph{not} optimal.

\begin{example}\label{ex.9}
Consider $n$ independent draws from a distribution in which $90\%$ of the domain elements occur with probability $10/n$, and the remaining $10\%$ occur with probability $11/n$.  If one assigns probability $11/n$ to the $10\%$ of the domain elements with largest empirical frequencies, the $\ell_1$ distance will be roughly $0.2$, because the vast majority of the elements with the largest empirical frequencies will actually have true probability $10/n$ rather than $11/n$.   In contrast,  if one ignores the samples and simply assigns probability $10/n$ to all the domain elements, the $\ell_1$ error will be exactly $0.1$.
\end{example}

This optimization task of assigning probabilities $q(x)$ (as a function of the true probabilities $v$ and set of $n$ samples) so as to minimize the expected $\ell_1$ error is well-defined. Nevertheless, this task seems to be computationally intractable.   Part of the computational challenge is that the optimal probability to assign to a domain element $x$ might be a function of $v$, the number of occurrences of $x$ in the sample, and also the number of occurrences of all the other domain elements.  Nevertheless, we describe a very natural and computationally efficient scheme which assigns a probability $q(x)$ to each $x$ that is a function of only $v$ and the number of occurrences of $x$, and we show that this scheme incurs an expected error within $o(1)$ of the expected error of the optimal scheme (which assigns $q(x)$ as a function of $v$ and the entire set of samples).     Of course, there is the additional complication that, in the context of our two step algorithm, we do not actually have the exact vector of true probabilities---only an approximation of such a vector---and hence this second phase of our algorithm must be robust to some noise in the recovered probabilities.

\medskip

Beyond yielding a near optimal learning algorithm, there are several additional benefits to our approach of first accurately reconstructing the unlabeled vector of probabilities.  For instance, such an unlabeled vector allows us to estimate properties of the underlying distribution including estimating the error of our returned vector, and estimating the error in our estimate of \emph{each} observed domain element's probability.  Additionally, as the following proposition quantifies, this unlabeled vector of probabilities can be leveraged to produce an accurate estimate of the expected number of distinct elements that will be observed in sample sizes up to a logarithmic factor larger:

\begin{proposition}\label{prop:extrapolate_intro}
Given $n$ samples from an arbitrary distribution $p$, with high probability over the randomness of the samples, one can estimate the expected number of unique elements that would be seen in a set of $k$ samples drawn from $p$, to within error $k\cdot c\sqrt{\frac{k}{n\log n}}$ for some universal constant $c$.
\end{proposition}

This proposition is tight, and it is slightly surprising in that the \emph{factor} by which one can accurately extrapolate increases with the sample size.  This ability to accurately predict the expected number of new elements observed in larger sample sizes is especially applicable to such settings as genomics, where data is relatively costly to gather, and the benefit of data acquisition is largely dependent on the number of new phenomena discovered.\footnote{One of the medical benefits of ``genome wide association studies'' is the compilation of catalogs of rare mutations that occur in healthy individuals; these catalogs are being used to rule out genetic causes of illness in patients, and help guide doctors to accurate diagnoses (see e.g.~\cite{10002015global,macarthur2012systematic}). Understanding how these catalogs will grow as a function of the number of genomes sequenced may be an important factor in designing such future datasets.}

\subsection{Related Work}

Perhaps the first work on correcting the empirical distribution is the work of Turing, and I.J. Good~\cite{Turing} (see also~\cite{GT56})---which serves as the jumping-off point for nearly all of the subsequent work on this problem that we are aware.   In the context of their work at Bletchley Park as part of the British WWII effort to crack the German enigma machine ciphers, Turing and Good developed a simple estimator that corrected the empirical distribution, in some sense to capture the ``missing'' probability mass of the distribution.   This estimator and its variants have been employed widely, particularly in the contexts of genomics, natural language processing, and other settings in which significant portions of the distribution are comprised of domain elements with small probabilities (e.g. ~\cite{CG99}).  In its most simple form, the Good-Turing frequency estimation scheme estimates the total probability of all domain elements that appear exactly $i$ times in a set of $n$ samples as $\frac{(i+1) \FF_{i+1}}{n},$ where $\FF_j$ is the total number of species that occur exactly $j$ times in the samples.  The total probability mass consisting of domain elements that are not seen in the samples---the ``missing'' mass, or, equivalently, the probability that the next sample drawn will be a new domain element that has not been seen previously---can be estimated by plugging $i=0$ into this formula to yield $\FF_{1}/n,$ namely the fraction of the samples consisting of domain elements seen exactly once.

	The Good--Turing estimate is especially suited to estimating the total mass of elements that appear few times; for more frequently occurring domain elements, this estimate has high variance---for example, if $\FF_{i+1} = 0,$ as will be the case for most large $i$, then the estimate is $0$.   However, for frequently occurring domain elements, the empirical distribution will give an accurate estimate of their probability mass.  There is an extremely long and successful line of work, spanning the past 60 years, from the computer science, statistics, and information theory communities,  proposing  approaches to ``smoothing'' the Good--Turing estimates, and combining such smoothed estimates with the empirical distribution (e.g.~\cite{GT56,GS95,MS00,AGT1,AGT2,DM04,AJOS13}).
	
	Our approach---to first recover an estimate of the unlabeled vector of probabilities of the true distribution and then assign probabilities to the observed elements informed by this recovered vector of probabilities---deviates fundamentally from this previous line of work.  This previous work attempts to accurately estimate the total probability mass corresponding to the set of domain elements observed $i$ times, for each $i$.  Even if one knows these quantities \emph{exactly}, such knowledge does not translate into an optimal learning algorithm, and could result in an $\ell_1$ error that is a factor of two larger than that of our approach.  The following rephrasing of Example~\ref{ex.9} from above illustrates this point:
	
\begin{example}
Consider $n$ independent draws from a distribution in which $90\%$ of the domain elements occur with probability $10/n$, and the remaining $10\%$ occur with probability $11/n$.  All variants of the Good-Turing frequency estimation scheme would end up, at best, assigning probability $10.1/n$ to most of the domain elements, incurring an $\ell_1$ error of roughly $0.2$.  This is because, for elements seen roughly $10$ times, the scheme would first calculate that the \emph{average} mass of such elements is $10.1/n,$ and then assign this probability to all such elements.   Our scheme, on the other hand, would realize that approximately $90\%$ of such elements have probability $10/n$, and $10\%$ have probability $11/n$, but then would assign the probability minimizing the expected error---namely, in this case, our algorithm would assign the \emph{median} probability, $10/n$, to all such elements, incurring an $\ell_1$ error of approximately $0.1$.
\end{example}

\noindent \textbf{Worst-case vs. Instance Optimal Testing and Learning.}
Sparked by the seminal work of Goldreich, Goldwasser and Ron~\cite{GGR} and that of Batu et al.~\cite{Ronitt-2000,independence}, there has been a long line of work considering distributional property testing, estimation, and learning questions from a \emph{worst case} standpoint---typically parameterized via an upper bound on the support size of the distribution from which the samples are drawn (e.g.~\cite{monotonicity,support,entropy,ent2,CCM,Paninski2,V-sicomp,FOCSVV,stocVV,CDVV14,nipsVV}).

The desire to go beyond this type of worst-case analysis and develop algorithms which provably perform better on ``easy''  distributions has led to two different veins of further work.    One vein considers different common types of structure that a distribution might possess--such as monotonicity, unimodality, skinny tails, etc., and how such structure can be leveraged to yield improved algorithms~\cite{newD,BS07,DKN15}.  While this direction is still within the framework of worst--case analysis, the emphasis is on developing a more nuanced understanding of why ``easy'' instances are easy.

Another vein of very recent work beyond worst-case analysis (of which this paper is an example) seeks to develop ``instance--optimal'' algorithms that are capable of exploiting whatever structure is present in the instance.  For the problem of identity testing---given the explicit description of $p$, deciding whether a set of samples was drawn from $p$ versus a distribution with $\ell_1$ distance at least $\eps$ from $p$---recent work gave an algorithm and an explicit function of $p$ and $\eps$ that represents the sample complexity of this task, for each $p$~\cite{instanceopt}.  In a similar spirit, with the dual goals of developing optimal algorithms as well as understanding the fundamental limits of when such instance--optimality is not possible, Acharya et al. have a line of work from the perspective of competitive analysis~\cite{ADJOP11,ADJOPS12,AJOS13a,AJOS13}.  Broadly, this work explores the following question: to what extent can an algorithm perform as well as if it knew, a priori, the structure of the problem instance on which it was run?  For example, the work~\cite{ADJOPS12} considers the two-distribution identity testing question: given samples drawn from two unknown distributions, $p$ and $q$, how many samples are required to distinguish the case that $p=q$ from $||p-q||_1 \ge \eps$?   They show that if $n_{p,q}$ is the number of samples required by an algorithm that knows, ahead of time, the unlabeled vector of probabilities of $p$ and $q$, then the sample complexity is bounded by $n_{p,q}^{3/2}$, and that, in general, a polynomial blowup is necessary---there exists $p,q$ for which no algorithm can perform this task using fewer than $n_{p,q}^{7/6}$ samples.
\medskip

\noindent \textbf{Relation to~\cite{stocVV,nipsVV}.}
The papers~\cite{stocVV,nipsVV} were concerned with developing estimators for entropy, support size, etc.---properties that depend only on the unlabeled vector of probabilities of a distribution.   The goal in those papers was to give tight \emph{worst-case} bounds on these estimation tasks in terms of a given upper bound on the support size of the distribution in question.  In contrast, this work considers the question of \emph{learning} probabilities for each labeled domain element, and considers the more ambitious and practically relevant goal of ``instance-optimality''.   This present paper has two technical components corresponding to the two stages of our algorithm: the first component is recovering an approximation to the unlabeled vector of probabilities, and the second part leverages the recovered unlabeled vector of probabilities to output a labeled vector.  The majority of the technical machinery that we employ for the first part is based on algorithms and techniques developed in~\cite{stocVV,nipsVV}, though analyzed here in a more nuanced and general way (a main tool from these works is a Chebyshev polynomial earthmover scheme, which was also repurposed for a rather different end in~\cite{FOCSVV}; the main improvement in the analysis is that our results no longer require any bound on the support size of the distribution, and the results no longer degrade with increasing support size).
We are surprised and excited that these techniques, originally developed for establishing worst-case optimal algorithms for property estimation can be fruitfully extended to yield tight instance-optimal results for such a fundamental and classic learning question.

\subsection{Definitions}
We refer to the unlabeled vector of probabilities of a distribution as the \emph{histogram} of the distribution.  This is simply the histogram, in the usual sense of the word, of the vector of probabilities of the domain elements.   We give a formal definition:

\begin{definition}~\label{def:histogram}
The \emph{histogram} of a distribution $p$, with a finite or countably infinite support,  is a mapping $h_p: (0,1] \rightarrow \mathbb{N}\cup \{0\},$ where $h_p(x)$ is equal to the number of domain elements that occur in distribution $p$ with probability $x$.  Formally,  $h_p(x)=|\{\alpha: p(\alpha)=x\}|$, where $p(\alpha)$ is the probability mass that distribution $p$ assigns to domain element $\alpha.$  We will also allow for ``generalized histograms'' in which $h_p$ does not necessarily take integral values.
\end{definition}

In analogy with the histogram of a distribution, it will be convenient to have an unlabeled representation of the set of samples.  We define the \emph{fingerprint} of a set of samples, which essentially removes all the label-information.  We note that in some of the literature, the fingerprint is alternately termed the \emph{pattern}, \emph{histogram}, \emph{histogram of the histogram} or \emph{collision statistics} of the samples.

\begin{definition}~\label{def:fingerprint}
  Given samples $X=(x_1,\ldots, x_n)$, the associated \emph{fingerprint},  $\FF=(\FF_1,\FF_2,\ldots)$, is the ``histogram of the histogram'' of the sample.  Formally, $\FF$ is the vector whose $i^{th}$ component, $\FF_i$, is the number of elements in the domain that occur exactly $i$ times in $X$.  \end{definition}

 The following earthmover metric will be useful for comparing histograms.  This metric is similar to that leveraged in~\cite{stocVV}, but allows for discrepancies in sufficiently small probabilities to be suppressed.  This turns out to be the ``right'' metric for establishing our learning result, as well as our result for the accurate estimation of the expected number of distinct elements that will be observed for larger sample sizes (Proposition~\ref{prop:extrapolate_intro}).  In both of these settings, we do not need to worry about accurately estimating extremely small probabilities, as long as we can accurately estimate the total aggregate probability mass comprised of such elements.

\begin{definition}\label{def:rem}
For two distributions $p_1, p_2$ with respective histograms $h_1,h_2,$ and a real number $\tau \in [0,1]$, we define the \emph{$\tau$-truncated relative earthmover distance} between them,  $R_{\tau}(p_1,p_2)\defeq R_{\tau}(h_1,h_2)$,  as the minimum over all schemes of moving the probability mass in the first histogram to yield the second histogram, where the cost per unit mass of moving from probability $x$ to probability $y$ is $|\log\max(x,\tau)- \log \max(y,\tau)|$.
\end{definition}

The following fact, whose proof is contained in Appendix~\ref{ap:fact1}, relates the $\tau$-truncated relative earthmover distance between two distributions, $p_1,p_2,$ to an analogous but weaker statement about the $\ell_1$ distance between $p_1$ and a distribution obtained from $p_2$ by choosing an optimal relabeling of the support:

\begin{fact}\label{fact:em2l1}
Given two distributions $p_1,p_2$ satisfying $R_{\tau}(p_1,p_2) \le \eps,$  there exists a relabeling $\pi$ of the support of $p_2$ such that $\sum_{i}\left| \max(p_1(i),\tau)-\max(p_2(\pi(i)),\tau)\right| \le 2\eps.$
\end{fact}

The Poisson distribution will also feature prominently in our algorithms and analysis:
\begin{definition}\label{def:poisson}
For $\lambda\geq 0$, we define $Poi(\lambda)$ to be the Poisson distribution of parameter $\lambda$, where the probability of drawing $j\leftarrow Poi(\lambda)$ equals $poi(\lambda,j)=\frac{e^{-\lambda}\lambda^j}{j!}$.
\end{definition}

\section{Recovering the histogram}\label{sec:alg1}

This section adapts the techniques of~\cite{stocVV,nipsVV} to accurately estimate the histogram of the distribution in a form that will be useful for Algorithm~\ref{def:alg2}, our ultimate instance-optimal algorithm for learning the distribution, presented and analyzed in Section~\ref{sec:disambig}.

The first phase of our algorithm, the step in which we recover an accurate approximation of the histogram of the distribution from which the samples were drawn, consists of solving an intuitive linear program.  The variables of the linear program represent the histogram entries $h(x_1),h(x_2),\ldots$ corresponding to a fine discretization of the set of probability values $0<x_1<x_2<\ldots < 1.$  The constraints of the linear program represent the fact that $h$ corresponds to the histogram of a distribution, namely all the probabilities sum to 1, and the histogram entries are non-negative.  Finally, the objective value of the linear program attempts to ensure that the histogram $h$ output by the linear program will have the property that, if the samples had been drawn from a distribution with histogram $h$, the expected number of domain elements observed once, twice, etc., would closely match the corresponding actual statistics of the sample.  Namely, the objective function tries to ensure that the expected fingerprint of the histogram returned by the linear program is close to the actual fingerprint of the samples.

One minor subtlety is that we will only solve this linear program for the portion of the histogram corresponding to domain elements that are not seen too many times.  For elements seen very frequently (at least $n^{\alpha}$ times for some appropriately chosen absolute constant $\alpha > 0$) their empirical probabilities are likely quite accurate, and we simply use these probabilities.  A similar approach was fruitfully leveraged in~\cite{stocVV,nipsVV} with the goal of creating worst-case optimal estimators for entropy, and other distributional properties of interest, and a related heuristic was proposed in the 1970's by Efron and Thisted~\cite{ET76}, also with the goal of estimating properties of the underlying distribution.

\medskip

We state the algorithm and its analysis in terms of two positive constants, $\constB,\constC,$ which can be defined arbitrarily provided the following inequalities hold:
$0.1 > \constB > \constC > \frac{\constB}{2}>0.$

\begin{center}
\myalg{def:alg1}{~}{
\textbf{Input:} Fingerprint $\FF$ obtained from $n$-samples.\\
\textbf{Output:} Histogram $h_{LP}.$
\begin{itemize}
\item{Define the set $X \defeq \{\frac{1}{n^2},\frac{2}{n^2},\frac{3}{n^2},\ldots,\frac{n^\constB+n^\constC}{n}\}.$}
\item{For each $x \in X,$ define the associated variable $v_x$, and solve the LP:
\vspace{-.1cm}
$$\text{Minimize  } \sum_{i=1}^{n^\constB} \left| \FF_i - \sum_{x \in X} poi(nx,i)\cdot v_x\right|$$
\qquad \qquad \qquad \qquad \quad Subject to: \vspace{.2cm}\\
\vspace{.2cm}
\indent   \qquad  $\cdot$\quad $\sum_{x \in X} x \cdot v_x +\sum_{i>n^\constB+2n^\constC}^n \frac{i}{n}\FF_i =1$  (total prob. mass $= 1$)\\
\vspace{.2cm}
\indent   \qquad  $\cdot $\quad$\forall x \in X,  v_x \ge 0$ (histogram entries are non-negative)}
\item{Let $h_{LP}$ be the histogram formed by setting $h_{LP}(x_i)=v_{x_i}$ for all $i$, where $(v_x)$ is the solution to the linear program, and then for each integer $i > n^\constB+2n^\constC$, incrementing $h_{LP}(\frac{i}{n})$ by $\FF_i$.}
\end{itemize}}
\end{center}

The following theorem quantifies the performance of the above algorithm:

\begin{theorem}\label{thm:h2}
There exists an absolute constant $c$ such that for sufficiently large $n$ and any $w \in [1,\log n],$ given $n$ independent draws from a distribution $p$ with histogram $h$, with probability $1-e^{-n^{\Omega(1)}}$ the generalized histogram $h_{LP}$ returned by Algorithm~\ref{def:alg1} satisfies $$R_{\frac{w}{n \log n}}(h,h_{LP}) \le \frac{c}{\sqrt{w}}.$$
\end{theorem}

This theorem is a stronger and more refined version of the results in~\cite{stocVV}, in that these results no longer require any bound on the support size of the distribution, and the results no longer degrade with increasing support size.  Instead, we express our results in terms of a lower bound, $\tau$, on the probabilities that we are concerned with accurately reconstructing.  We provide the proof of the theorem in a self-contained fashion in Appendix~\ref{ap:thm2}.

\medskip

One interpretation of Theorem~\ref{thm:h2} is that Algorithm~\ref{def:alg1}, when run on $n$ independent draws from a distribution, will accurately reconstruct the histogram, in the relative earthmover sense, all the way down to probability $\frac{1}{n\log n}$ (significantly below the $1/n$ threshold where the empirical distribution becomes ineffective). One corollary of independent interest is that this earthmover bound implies that we can accurately extrapolate the number of unique elements that will be seen if we run a new, larger experiment, of size up to $n\log n$. Given a histogram $h$, for each element of probability $x$, the probability that it will be seen (at least once) in a sample of size $k$ equals $1-(1-x)^k$; thus, the expected number of unique elements seen in a sample of size $k$ for a distribution with histogram $h$ equals $\sum_{x:h(x)\neq 0} (1-(1-x)^k)\cdot h(x)$.
The following lemma, whose proof is given in Appendix~\ref{appendix:lem:unique}, shows that this quantity is Lipschitz continuous with respect to truncated relative earthmover distance. 

\begin{lemma}\label{lem:unique}
Given two (possibly generalized) histograms $g,h$, a number of samples $k$, and a threshold $\tau\in(0,1]$,
\[\left|\sum_{x:g(x)\neq 0} (1-(1-x)^k)\cdot g(x)-\sum_{x:h(x)\neq 0} (1-(1-x)^k)\cdot h(x)\right|\leq (0.3(k-1)+1)R_\tau(g,h)+\tau \frac{k}{2}\]

\end{lemma}

The above lemma together with Theorem~\ref{thm:h2} yields Proposition~\ref{prop:extrapolate_intro},  which
is tight, in the sense that one cannot expect meaningful extrapolation beyond sample sizes of $n\log n$, as shown by the lower bounds in~\cite{stocVV}.\footnote{Namely, for some constant $c$, there exist two families of distribution, $\mathcal{D}_1$ and $\mathcal{D}_2$ such that the distributions in $\mathcal{D}_1$ are close to uniform distributions on $cn\log n$ elements, and the distributions in $\mathcal{D}_2$ are close to uniform distributions over $2c n\log n$ elements, yet a randomly selected distribution in $\mathcal{D}_1$ is (with  constant probability) information theoretically indistinguishable from a randomly selected distribution in $\mathcal{D}_2$ if one is given only $n$ samples drawn from the distributions.  Of course, given $2c n \log n$ samples, the number of distinct elements observed will likely be either $\approx (1-\frac{1}{e^2}) c n \log n \approx 0.9 c n \log n$ or $\approx (1-\frac{1}{e})2c n \log n \approx 1.3 c n \log n$ according to the two cases.}

\medskip

Towards our goal of devising an optimal learning algorithm, the following corollary of Theorem~\ref{thm:h2} formalizes the sense that the quality of the histogram output by Algorithm~\ref{def:alg1} will be sufficient to achieve our optimal learning result, provided that the second phase of our algorithm, described in Section~\ref{sec:disambig}, is able to successful label the histogram.

\begin{corollary}\label{prop:hist}
There exists an algorithm such that, for any function $f(n) = \omega_n(1)$ that goes to infinity as $n$ gets large (e.g. $f(n)=\log \log n$), there is a function $o_n(1)$ of $n$ that goes to zero as $n$ gets large, such that given $n$ samples drawn independently from any distribution $p$, the algorithm outputs an unlabeled vector, $q$, such that, with probability $1-e^{-n^{\Omega(1)}}$, there exists a labeling $\pi(q)$ of the vector $q$ such that $$\sum_{i}\left| \max\left(p(x),\frac{f(n)}{n \log n}\right)-\max\left(\pi(q)(x),\frac{f(n)}{n \log n}\right)\right| < o_n(1),$$ where $p(x)$ denotes the true probability of domain element $x$ in distribution $p$.
\end{corollary}

This corollary is not immediate: the histogram returned by the algorithm might be non-integral, however in Appendix~\ref{sec:round} we provide a simple algorithm that rounds a generalized histogram to an (integral) histogram, while changing it very little in relative earthmover distance $R_0(\cdot,\cdot).$  This rounding, together with Fact~\ref{fact:em2l1}, obtains this corollary.

The utility of the above corollary lies in the following observation:  for any function $g(n)=o(1/n),$ the domain elements $x$ that both occur in the $n$ samples and have true probability $p(x) < g(n)$, can account for at most $o(1)$ probability mass, in aggregate.  In other words, while the true distribution might have a constant amount, $c$, of probability mass consisting of domain elements that occur with probability $o(1/n)$, we would observe at most a $o(1)$ fraction of such domain elements in the $n$ samples.  Hence, even an optimal scheme that knows the true probabilities would be unable to achieve an $\ell_1$ error less than $c-o(1)$ because it does not know the labels of the elements that have not been observed, and we could also hope to achieve an $\ell_1$ error of roughly $c$.

\section{Disambiguating the histogram} \label{sec:disambig}
\newcommand{\ccc}{{2}}
\newcommand{\twiceccc}{{4}}
\newcommand{\ddd}{{0.1}}
\newcommand{\eee}{{0.25}}

In this section we present our instance-optimal algorithm for learning a distribution from $n$ samples, making use of Algorithm~\ref{def:alg1} of Section~\ref{sec:alg1} to first accurately infer the histogram of the distribution (in the sense of Corollary~\ref{prop:hist}).  As a motivating intuition for the second phase of our algorithm--the phase in which we assign probabilities to the observed elements---consider the behavior of an optimal algorithm that not only knows the true histogram $h$ of the distribution, but also knows for each positive integer $j$ the entire multiset of probabilities of elements that appear exactly $j$ times in the $n$ samples. Since the algorithm has no basis to distinguish between the different elements that each appear $j$ times in the samples, the algorithm may as well assign a single probability $m_j$ to all the items that appear $j$ times in the samples. The optimal $m_j$ in this setting is easily seen to be the \emph{median} of the multiset of probabilities of items appearing $j$ times, as the median is the estimate that minimizes the total ($\ell_1$) error of the probabilities.

Our algorithm aims to emulate this idealized optimal algorithm.  Of course, we must do this using only an estimate of the histogram, and computing medians based on the likelihoods that elements of probability $x$ will be seen $j$ times in the sample, as opposed to actual knowledge of the multiset of probabilities of the elements observed $j$ times (which was an unreasonably strong assumption, that we made in the previous paragraph because it let us argue about the behavior of the optimal algorithm in that case).

Because our algorithm needs to work in terms of a histogram estimate $u$, bounded only by the guarantees of Corollary~\ref{prop:hist}, we add an additional ``regularization" step that was not needed in the idealized medians setting described above.  We ``fatten" the histogram $u$ to a new histogram $\bar{u}$ by adding a small amount of probability mass across the range $[\frac{1}{n},\frac{1}{n}\log^\ccc n]$, which acts to mollify the effect on the medians of any small errors in the histogram estimate.

Given this ``fattened'' approximate histogram, we then apply the ``medians'' intuition to computing, for each integer $j$, an appropriate probability with which to label those elements occurring $j$ times in the sample. These estimates are computed via the following thought experiment: imagining $\bar{u}$ to be the true histogram, if we take $n$ samples from the corresponding distribution, \emph{in expectation}, what is the median of the (multiset) of probabilities of those elements seen exactly $j$ times in the sample? We denote this ``expected median'' by $m_{\bar{u},j,n}$, and our algorithm assigns this probability to each element seen $j$ times in the sample, for $j<\log^\ccc n$, and assigns the empirical probability $\frac{j}{n}$ for larger $j$.  We formalize this process with the following definition for ``Poisson-weighted medians'':

\begin{definition}\label{def:median}
  Given a histogram $h$, let $S_h$ be the multiset of probabilities of domain elements---that is, for each probability $x$ for which $h(x)$ is some positive integer $i$, add $i$ copies of $x$ to $S$. Given a number of samples $n$, and an index $j$, consider weighting each element $x\in S_h$ by $poi(nx,j)$. Define $m_{h,j,n}$ to be the median of this weighted multiset.
\end{definition}

Explicitly, the median of a \emph{weighted} set of real numbers is a number $m$ such that at most half the weight lies on numbers greater than $m$, and at most half lies on numbers less than $m$.  Taking advantage of the medians defined above, our learning algorithm follows:

\begin{center}
\myalg{def:alg2}{~}{
\textbf{Input:} $n$ samples from a distribution $h$.\\
\textbf{Output:} An assignment of a probability to each nonzero entry of $h$.
\begin{itemize}
\item Run Algorithm~\ref{def:alg1} to return a histogram $u$.
\item Modify $u$ to create $\bar{u}$ by, for each $j\leq\log^{\ccc} n$ adding $\frac{n}{j\log^{ \twiceccc} n}$ elements of probability $\frac{j}{n}$ and removing corresponding mass arbitrarily from the rest of the distribution.
\item Then to each fingerprint entry $j<\log^{\ccc} n$, assign those domain elements probability $m_{\bar{u},j,n}$, (as defined in Definition~\ref{def:median}) and to each higher fingerprint entry $j\geq\log^{\ccc}n$ assign those domain elements their empirical probability $\frac{j}{n}$.
\end{itemize}
}
\end{center}


\medskip

\noindent \textbf{Theorem~\ref{thm:main}} \emph{
There is a function $err(n)$ that goes to zero as $n$ gets large, such that Algorithm~\ref{def:alg2}, when given as input $n$ independent draws from any distribution $p$ of discrete support, outputs a labeled vector $q$, such that $$\E\left[||p-q||_1\right] \le opt(p,n)+err(n),$$ where $opt(p,n)$ is the minimum expected error that any algorithm could achieve on the following learning task: given $p$, and given $n$ samples drawn independently from a distribution that is identical to $p$ up to an arbitrary relabeling of the domain elements, learn the distribution.
}

\medskip

The core of the proof of Theorem~\ref{thm:main} relies on constructing an estimate, $dev_{j,n}(A,m_{B,j,n}),$ that captures the expected contribution to the $\ell_1$ error due to elements that occur exactly $j$ times, given that the true distribution we are trying to reconstruct has histogram $A$, and our reconstruction is based on the medians $m_{B,j,n}$ derived from a (possibly different) histogram $B$.  The proof then has two main components. First we show that $dev_{j,n}(h,m_{h,j,n})$ approximately captures the performance of the optimal algorithm with very high probability, namely that using the true histogram $h$ to choose medians $m_{h,j,n}$ lets us estimate the performance of the best possible algorithm.   This step is slightly subtle, and implies that an algorithm that knows $h$ can glean at most $o(1)$ added benefit by computing the probability assigned to an element that occurs $j$ times using a function that depends on $j$ and $h$ and the entire set of samples, rather than just $j$ and $h$.

Next, we show that the clean functional form of this ``median'' estimate implies that $dev(\cdot,\cdot)$ varies slowly with respect to changes in the second histogram (used to choose the median in the second term), and thus that with only negligible performance loss we may reconstruct distributions using medians derived from an \emph{estimate} $u$ of the true histogram, thus allowing us to analyze the actual performance of Algorithm~\ref{def:alg2}.

Beyond these two core steps, the analysis of Algorithm~\ref{def:alg2} is somewhat delicate---because our algorithm is instance-optimal to $o(1)$ error, it must reuse samples both for the Algorithm~\ref{def:alg1} histogram reconstruction and for the final labeling step, and we must carefully separate the probabilistic portion of the analysis via a clean set of assumptions which 1) will hold with near certainty over the sampling process, and 2) are sufficient to guarantee the performance of both stages of our algorithm.   The complete proof is contained in Appendix~\ref{ap:pfm}.

\bibliographystyle{plain}
\bibliography{aps}

\appendix

\section{Proof of Theorem~\ref{thm:main}}\label{ap:pfm}
In this section we give a self-contained proof of the correctness of Algorithm~\ref{def:alg2}, establishing Theorem~\ref{thm:main}.  

\subsection{Separating out the probabilistic portion of the analysis}\label{sec:faithful1}

Our analysis is somewhat delicate because we reuse the same samples both to estimate the histogram $h$, and then to label the domain elements given an approximate histogram. For this reason, we will very carefully separate out the probabilistic portion of the sampling process, identifying a list of convenient properties which happen with very high probability in the sampling process, and then deterministically analyze the case when these properties hold, which we will refer to as a ``faithful" set $S$ of samples from the distribution. (Appendix~\ref{ap:thm2} uses this same analysis technique, though with a different notion of ``faithful", appropriate for the different desiderata of that appendix.)

We first describe a simple discretization of histograms $h$, dividing the domain into buckets which will simplify further analysis, and is a crucial component of the definition of ``faithful".

\begin{definition}\label{def:buckets}
  Given a histogram $h$, and a number of samples $n$, define the $k$th \emph{bucket} of $h$ to consist of those histogram entries with probabilities in the half-open interval $(\frac{k}{n\log^{\ccc} n},\frac{k+1}{n\log^{\ccc} n}]$. Letting $h_k$ be $h$ restricted to its $k$th bucket, define $B_{poi}(j,k)=\sum_{x:h_k(x)\neq 0} h(x) poi(nx,j)$ to be the expected number of elements from bucket $k$ that are seen exactly $j$ times, if $Poi(n)$ samples are taken. Given a set of samples $S$, let $B_S(j,k)$ be the number of elements in bucket $k$ of $h$ that are seen exactly $j$ times in the samples $S$, where in both cases $h$ and $n$ are implicit in the notation.
\end{definition}

Given this notion of ``buckets", we define faithful to mean 1) each domain element is seen roughly the number of times we would expect to see it, and 2) for each pair $(j,k)$, the number of domain elements from bucket $k$ that are seen exactly $j$ times is very close to its expectation (where we compute expectations under a Poisson distribution of samples, because ``Poissonization" will simplify subsequent analysis). The first condition of ``faithful" gives weak control on which fingerprint entry each domain element will contribute to, while the second condition gives much stronger control over the aggregate contribution to fingerprint entries by all domain elements within a certain probability ``bucket".

\begin{definition}\label{def:faithful2}
  Given a histogram $h$ and a number of samples $n$, a set of $n$ samples, $S$, is called \emph{faithful} if:
  \begin{enumerate}
    \item Each item of probability $x$ appears in the samples a number of times $j$ satisfying $|nx-j|<\max\{\log^{1.5} n,\sqrt{nx\log^{1.5} n}\}$, and
    \item For each $j<\log^{\ccc} n$ and $k$, we have $|B_{poi}(j,k)-B_S(j,k)|< n^{0.6}$.
  \end{enumerate}
\end{definition}

This notion of ``faithful" holds with near certainty, as shown in the following lemma, allowing us to assume (when specified) in the results in the rest of this section that our learning algorithm receives a faithful set of samples.

\begin{lemma}\label{lem:faithful2}
  For any histogram $h$ and number of samples $n$, with probability $1-n^{-\omega(1)}$, a set of $n$ samples drawn from $h$ will be faithful.
\end{lemma}
\begin{proof}
Since the number of times an item of probability $x$ shows up in $n$ samples is the binomial distribution $Bin(n,x)$, the first condition of ``faithful"---essentially that this random variable will be within $\log^{3/4} n$ standard deviations of its mean--- follows with probability $1-n^{-\omega(1)}$ from standard Chernoff/Hoeffding bounds.


For the second condition, since $Poi(n)$ has probability $\Theta(1/\sqrt{n})$ of equaling $n$, we consider the related process where $Poi(n)$ samples are drawn. The number of times each domain element $x$ is seen is now distributed as $Poi(nx)$, independent of each other domain element. Thus the number of elements from bucket $k$ seen exactly $j$ times is the sum of independent Bernoulli random variables, one for each domain element in bucket $k$. The expected number of such elements is $B_{poi}(j,k)$ by definition. Since $B_{poi}(j,k)\leq n$ by definition, we have that the variance of this random variable is also at most $n$, and thus Chernoff/Hoeffding bounds imply that the probability that it deviates from its expectation by more than $n^{0.6}$ is at most $exp(-n^{0.1})$. Thus the probability of such a deviation is at most a $\Theta(\sqrt{n})$ factor higher when taking exactly $n$ samples than when taking $Poi(n)$ samples; taking a union bound over all $j$ and $k$ yields the desired result.
\end{proof}

\subsection{An estimate of the optimal error}\label{sec:opt-est}

We now introduce the key definition of $dev(\cdot,\cdot)$, which underpins our analysis of the error of estimation algorithms. The definition of $dev(\cdot,\cdot)$ captures the following process: Suppose we have a probability value $m_j$, and will assign this probability value to every domain element that occurs exactly $j$ times in the samples.  We estimate the expected error of this reconstruction, in terms of the probability that each domain element shows up exactly $j$ times. While the below definition, stated in terms of a Poisson process, is neither clearly related to the optimal error $opt(h,n)$, nor the actual error of any specific algorithm, it has particularly clean properties which will help us show that it can be related to both $opt(h,n)$ (in this subsection) as well as the expected error achieved by Algorithm~\ref{def:alg2} (shown in Section~\ref{sec:Lipschitz}).

\begin{definition}
  Given a histogram $h$, a real number $m$, a number of samples $n$, and a nonnegative integer $j$, define $dev_{j,n}(h,m)=\sum_{x:h(x)\neq 0} |x-m|h(x)poi(nx,j)$.
\end{definition}

Intuitively, $dev_{j,n}(h,m)$ describes the expectation---over taking $Poi(n)$ samples from $h$---of the sum of the deviations between $m$ and each probability $x$ of a element seen $j$ times among the samples. Namely, $dev_{j,n}(h,m)$ describes to what degree $m$ is a good probability to which we can ascribe all domain elements seen $j$ times, among $\approx n$ samples from $h$.

This definition provides crucial motivation for how Definition~\ref{def:median} sets the medians $m_{h,j,n}$ used in Algorithm~\ref{def:alg2}, since $m_{h,j,n}$ is the value of $m$ that minimizes the previous definition, $dev_{j,n}(h,m)$, since both are defined via the same Poisson weights $poi(nx,j)$. (The median of a---possibly weighted---set of numbers is the location $m$ that minimizes the total---possibly weighted---distance from the set to $m$.)

We now show the key result of this section, that the definition of ``faithful" induces precise guarantees on the spread of probabilities of those elements seen $j$ times. Subsequent lemmas will relate this to the performance of both the optimal algorithm and to our own Algorithm~\ref{def:alg2}.

\begin{lemma}\label{lem:opt-est}
  Given a histogram $h$, let $S$ be the multiset of probabilities of a faithful set of samples of size $n$. For each index $j<\log^{\ccc} n$, consider those domain elements that occur exactly $j$ times in the samples and let $S_j$ be the multiset of probabilities of those domain elements. Let $\sigma_j$ be the sum over $S_j$ of each element's distance from the median (counting multiplicity) of $S_j$. Then $\sum_{j<\log^{\ccc} n}|\sigma_j-dev_{j,n}(h,m_{h,j,n})|=O(\log^{-\ccc}n)$.
\end{lemma}

\begin{proof}

   Recall that $\sigma$ computes the total distance of the (unweighted) multiset $S_j$ from its median, while $dev_{j,n}(h,m_{h,j,n})$ is an analogous (weighted) quantity for the true histogram, with each entry $x$ having multiplicity $h(x)$ and weight $poi(nx,j)$. In the first case, sampling means that each element of probability $x$ either shows up exactly $j$ times (with some binomial probability) and is counted with weight 1, or does not show up $j$ times and is not counted; in the second case, instead of sampling, each entry $x$ from the histogram is counted with weight $poi(nx,j)<1$, capturing roughly the average effect of sampling (except with Poisson instead of binomial weight). By the definition of ``faithful", the total weight coming from each bucket $k$ in both cases is within $n^{0.6}$ of each other (since $j<\log^{\ccc} n$). We consider only buckets $k\leq 2\log^{2\cdot\ccc}n$, corresponding to probabilities less than $\frac{2}{n}\log^\ccc n$, since the first condition of ``faithful" means that no higher probability elements will be seen $j<\log^\ccc n$ times.


   Consider transforming one weighted multiset into the other (where elements of $S_j$ are interpreted as having weight 1 each), maintaining a bound on how much the total distance from the median changes. We make crucial use of the fact that the ``total distance to the median" is robust to small changes in the weighted multiset, since the median is the location that minimizes this total distance. Moving $\alpha$ weight by a distance of $\beta$ can increase the total (weighted) distance to the median by at most $\alpha\cdot \beta$ since this is how much the total weighted distance to the \emph{old} median changes, and the new median must be at least as good; conversely, such a move cannot decrease the total distance by more than $\alpha\cdot\beta$ as the inverse move would violate the previous bound. Adding $\alpha$ weight to the distribution at distance $\beta$ from the current median similarly cannot decrease the total distance, but also cannot increase the total distance by more than $\alpha\cdot\beta$, with the corresponding statements holding for removing $\alpha$ weight.

   Thus, transforming all the $S_j$ into their Poissonized analogs requires two types of transformations: 1) moving up to $n$ samples within their buckets; 2) adding or removing up to $n^{0.6}$ weight from buckets for various combinations of $j$ and $k$. Since buckets have width $1/(n\log^{\ccc}n)$, transformations of the first type change the total distance to the median by at most $\log^{-\ccc}n$; since $j<\log^{\ccc}n$ and all buckets above probability $\frac{2}{n}\log^{\ccc}n$ are empty, transformations of the second type change the total distance by at most the product of the weight adjustment $n^{0.6}$, the number of $j,k$ pairs $2\log^{3\cdot\ccc}n$, and size of the probability range under consideration which is $\frac{2}{n}\log^{\ccc}n$, yielding a bound of $\frac{4}{n^{0.4}}\log^{4\cdot\ccc}n$. Thus in total the change is $O(\log^{-\ccc}n)=o(1)$ as desired.
\end{proof}


The above lemma essentially shows that $dev_{j,n}(h,m_{h,j,n})$ captures how well we could hypothetically estimate the probabilities of all the domain elements seen $j$ times, under the unrealistically optimistic assumption that we know the (unlabeled) multiset of probabilities of elements seen $j$ times and estimate all these probabilities optimally by their median. Before showing how our algorithm can perform almost this well based on only the samples, we first formalize this reasoning.

\begin{definition}
  We call a distribution learner ``simple" if all the domain elements seen exactly $j$ times in the samples get assigned the same probability.
\end{definition}
Given $n$ samples from a distribution $p$, with $p_{(j)}$ being those domain elements that occurred exactly $j$ times in the sample, we note that the probability of obtaining these samples is invariant to any permutation of $p_{(j)}$. Thus if a hypothetical learner $L$ assigns different probabilities to different elements seen $j$ times in the sample, then its average performance over a random permutation of the domain elements can only improve if we simplify $L$, by having it instead assign to all the elements seen $j$ times the \emph{median} of the multiset that it was originally assigning.

For this reason, when we are discussing an optimal distribution learner, we will henceforth assume it is simple.

\begin{lemma}\label{lem:medians}
  Given a histogram $h$, let $S$ be the multiset of probabilities of a faithful set of samples of size $n$. Given an index $j<\log^{\ccc} n$, consider those domain elements that occur exactly $j$ times in the sample; let $S_j$ be the multiset of probabilities of those domain elements. Let $\sigma_j$ be the sum over $S_j$ of each element's distance from the median of $S_j$ (counting multiplicity). Then any simple learner, when given the sample, must have error at least $\sigma_j$ on the domain elements that appear $j$ times in the sample.
\end{lemma}

\begin{proof}
  The median of $S_j$ is the best possible estimate any simple learner can yield---even given the true distribution---so the error of this estimate bounds the performance of a simple learner.
\end{proof}

Combining this with Lemma~\ref{lem:opt-est} immediately yields:
\begin{corollary}\label{cor:opt-est}
  For any distribution $h$, the total error of any simple learning algorithm, given $n$ faithful samples from $h$, is at least $\left(\sum_{j<\log^{\ccc} n} dev_{j,n}(h,m_{h,j,n})\right)-O(\log^{-\ccc}n)$. Further, for any algorithm---simple or not---if we average its performance over all relabelings of the domain of $h$ and the corresponding relabeled samples, it will have expected error bounded by the same expression.
\end{corollary}

\subsection{Our error estimate is Lipschitz with respect to mis-estimating the distribution}\label{sec:Lipschitz}

We now relate the error bound of Corollary~\ref{cor:opt-est} to the performance of our algorithm, via two steps. The bound in the corollary is in terms of $m_{h,j,n}$, the medians computed in terms of the true histogram $h$ which is unknown to the algorithm; instead the algorithm works with an estimate $\bar{u}$ of the true histogram. The next lemma shows that estimating in terms of $\bar{u}$ is almost as good as using $h$.

\begin{fact}\label{fact:tail}
  For any distribution $h$, index $j\geq 1$, and real parameter $t\geq 1$, weighting each domain element $x$ by $poi(nx,j)$, the total weight on domain elements that are at least $t$ standard deviations away from $\frac{j}{n}$---namely, for which $|nx-j|\geq t\sqrt{j}$ is at most $n\cdot exp(-\Omega(t))$.
\end{fact}
\begin{lemma}\label{lem:lip}
  Given a number of samples $n$, a histogram $h$ and a second histogram $\bar{u}$ that is 1) close to $h$ in the sense of Corollary~\ref{prop:hist}, in that there exists distributions $p,q$ corresponding to $h,\bar{u}$ respectively for which $\sum_{i}| \max(p(i),\frac{1}{n}\log^{-\eee}n)-\max(q(i),\frac{1}{n}\log^{-\eee}n)| \le \log^{-\eee}n,$ and 2) the histogram $\bar{u}$ is ``fattened" in the sense that for each $j\leq\log^{\ccc} n$ there are at least $\frac{n}{j\log^{2\cdot \ccc} n}$ elements of probability $\frac{j}{n}$. Then $\sum_{j<\log^{\ccc}n}dev_{j,n}(h,m_{\bar{u},j,n})\leq o(1)+\sum_{j<\log^{\ccc}n}dev_{j,n}(h,m_{h,j,n})$.
\end{lemma}

Since for each $j$, as noted earlier, $m_{h,j,n}$ is the quantity which minimizes $dev_{j,n}(h,m)$, each term $dev_{j,n}(h,n_{\bar{u},j,n})$ on the left hand side is  greater than or equal to the corresponding $dev_{j,n}(h,n_{h,j,n})$ on the right hand side, so the lemma implies that the left and right hand sides of the expression in the lemma, beyond having related sums, are in fact term-by-term close to each other.

The proof relies on first comparing $m_{h,j}$ and $m_{\bar{u},j}$ to $\frac{j}{n}$, and then showing that $dev_j(h,m)$ is Lipshitz with respect to changes in $h$ of the type described by the guarantees of Corollary~\ref{prop:hist}.
\begin{proof}[Proof of Lemma~\ref{lem:lip}]
We drop the ``$,n$" subscripts here for notational convenience.

Recall that the quantities $m_{h,j}$ and $m_{\bar{u},j}$ are medians computed after weighting by a Poisson function centered at $j$, and thus we would expect these medians to be close to $\frac{j}{n}$. We first show that the ``fattening" condition makes $m_{\bar{u},j}$ well-behaved, and then show, given this, that the lemma works both in the case that $m_{h,j}$ is far from $\frac{j}{n}$, and then for the case where they are close.

By condition 2 of the lemma, the ``fattening" assumption, for any index $j<\log^{\ccc} n$, we have $\sum_{x:\bar{u}(x)\neq 0} h(x) poi(nx,j)=1/\log^{O(1)}n$. Thus, by Fact~\ref{fact:tail}, the median $m_{\bar{u},j}$ must satisfy $|n\cdot m_{\bar{u},j}-j|<\sqrt{j}\log^{o(1)}n$, since the fraction of the Poisson-weighted distribution that is at locations more than $\frac{1}{n}\sqrt{j}\log^{\Theta(1)}n$ distance from $\frac{j}{n}$ is (much) less than $1/2$.

Given the above bound on $m_{\bar{u},j}$, we now turn to $m_{h,j}$. Consider the case $|n\cdot m_{h,j}-j|>\sqrt{j}\log^{\ddd}n$. By Fact~\ref{fact:tail}, weighting each domain element $x$ by $poi(nx,j)$, the total weight on the far side of the median $m_{h,j}$ from $\frac{j}{n}$, is at most $n\cdot exp(-\Omega(\log^{\ddd}n))$. Since (by definition of ``median") half the weight is on each side of the median, the total weight $\sum_{x:h(x)\neq 0}h(x)poi(nx,j)$ must also be bounded by $n\cdot exp(-\Omega(\log^{\ddd}n))$. Recall the definition of the left hand side of the inequality of the lemma, $dev_{j}(h,m_{\bar{u},j})=\sum_{x:h(x)\neq 0} |x-m_{\bar{u},j}|h(x)poi(nx,j)$. Thus for the portion of this sum where $x<\frac{2}{n}\log^2 n$, since from the previous paragraph $m_{\bar{u},j}$ is also bounded by $\frac{2}{n}\log^2 n$ for large enough $n$, we can bound $\sum_{x<\frac{2}{n}\log^2 n:h(x)\neq 0} |x-m_{\bar{u},j}|h(x)poi(nx,j)$ by the product $n\cdot exp(-\Omega(\log^{\ddd}n))\cdot \frac{2}{n}\log^2 n=exp(-\Omega(\log^{\ddd}n))$. For those $x\geq \frac{2}{n}\log^2 n$, since $j<\frac{1}{n}\log^{\ccc} n$, we have the tail bounds $poi(nx,j)=n^{-\omega(1)}$, implying the total for such $x$ is also bounded by $exp(-\Omega(\log^{\ddd}n))$, which is our final bound for this case---summing these bounds over all $j<\log^2 n$ yields the desired bound $\sum_{j<\log^{\ccc}n}dev_{j}(h,m_{\bar{u},j})\leq o(1)$, where the sum is over those $j$ for which this case applies, $|n\cdot m_{h,j}-j|>\sqrt{j}\log^{\ddd}n$.

Thus it remains to prove the claim when both $m_{h,j}$ and $m_{\bar{u},j}$ are close to $\frac{j}{n}$. To analyze this case, we show that $dev_j(h,m)$ is Lipschitz with respect to the closeness in $h$ and $\bar{u}$ guaranteed by condition 1 of the lemma, provided $|n\cdot m-j|\leq\sqrt{j}\log^{\ddd} n$. The guarantee on $h$ and $\bar{u}$ means that one can transform one distribution into the other by two kinds of transformations: 1) changing the distributions by $\log^{-\eee}n$ in the $\ell_1$ sense, and 2) arbitrary mass-preserving transformation of elements of probability less than $\frac{1}{n}\log^{-\eee}n$. We thus bound the change in $dev_j(h,m)$ under both types of transformations.

To analyze $\ell_1$ modifications, consider an arbitrary probability $x$, and consider the derivative of $dev_j(h,m)$ as we take an element of probability $x$ and change $x$. Recalling the definition $dev_{j}(h,m)=\sum_{x:h(x)\neq 0} |x-m|h(x)poi(nx,j)$, we see that this derivative equals $\frac{d}{dx} |x-m|poi(nx,j)$, which is bounded (by the product rule and triangle inequality) as $poi(nx,j)+|x-m|\frac{d}{dx} poi(nx,j)$, where $\frac{d}{dx} poi(nx,j)=n\cdot poi(nx,j-1)\cdot (1-\frac{nx}{j})$. Rewriting $m$ as $m_j$ to indicate its dependence on $j$, we want to bound the sum of this derivative over $j<\log^{\ccc} n$, since the exact dependence for each individual $j$ is much harder to talk about than the overall dependence. We have $\sum_j poi(nx,j)+|x-m_j|n\cdot poi(nx,j-1)\cdot (1-\frac{nx}{j})$, where $\sum_j poi(nx,j)\leq 1$. To bound the remaining part of the sum, we first consider the case $x<\frac{1}{n}$, in which case we bound $|x-m_j|\leq \frac{1}{n}(1+j+\sqrt{j}\log^{\ddd}n)$ and $(1-\frac{nx}{j})\leq 1$, thus yielding the bound $\sum_{j\geq 1} |x-m_j|n\cdot poi(nx,j-1)\cdot (1-\frac{nx}{j})\leq \sum_{j\geq 0} (2+j+\sqrt{j+1}\log^{\ddd}n)poi(nx,j)\leq \sum_{j\geq 0} (2+j+\sqrt{j+1}\log^{\ddd}n)/j!=O(\log^{\ddd}n)$. For $x\geq \frac{1}{n}$, since $poi(nx,j-1)$ decays exponentially fast for $j$ more than $\sqrt{nx}$ away from $nx$, we can bound this sum  as being on the order of $\sqrt{nx}$ times its maximum value when $j$ is in this range. In this range we have $|x-m_j|\leq |x-\frac{j}{n}|+ |\frac{j}{n}-m_j|=\frac{1}{n}O(\sqrt{nx}\log^{\ddd}n)$, and $poi(nx,j-1)=O(\frac{1}{\sqrt{nx}})$, and $(1-\frac{nx}{j})=O(1/\sqrt{nx})$, yielding a total bound of $O(\sqrt{nx}\sqrt{nx}\frac{1}{\sqrt{nx}}\frac{1}{\sqrt{nx}}\log^{\ddd}n)=O(\log^{\ddd}n)$ as in the previous case. Thus we conclude that the sum over all $j$ of the amount $dev_j(h,m)$ changes with respect to $\ell_1$ changes in $h$ is $O(\log^{\ddd} n)$.

We next bound the total change to $dev_{j}(h,m)$ induced by the second type of modification, arbitrary mass-preserving transformations of elements of probability  $x<\frac{1}{n}\log^{-\eee}n$. For $j=1$, we bound the components of $dev_{j}(h,m)=\sum_{x:h(x)\neq 0} |x-m|h(x)poi(nx,j)$ by bounding the two terms in the product: $|x-m|\in [m-\frac{1}{n}\log^{-\eee}n,m+\frac{1}{n}\log^{-\eee}n]$, and $poi(nx,1)=nx\cdot e^{-nx}\in [nx(1-\log^{-\eee}n)^2,nx]$. Thus for $m$ either $m_h$ or $m_{\bar{u}}$, since by the assumption of this case $m\leq\frac{1}{n}(1+\log^{\ddd}n)$, from the bounds above, the contribution to $dev_j(h,m)$ from those $x<\frac{1}{n}\log^{-\eee}n$ is within $o(1)$ of $mn$ times the total mass in the distribution below $\frac{1}{n}\log^{-\eee}n$, showing that arbitrary modifications of the second type modify $dev_1(h,m)$ by $o(1)$.

 Analyzing the remaining $j\geq 2$ terms, omitting the $|x-m|$ multiplier for the moment, we have that $\sum_{x<\frac{1}{n} \log^{-\eee}n:h(x)\neq 0} h(x) poi(nx,j) \leq n(\log^{-\eee}n)^{j-1}$. Because of the bound that $m_h, m_{\bar{u}}$ are each within $\frac{1}{n}\sqrt{j}\log^{\ddd}n$ of $\frac{j}{n}$, we have that $|x-m|\leq \frac{1}{n}( \log^{-\eee}n+j+\sqrt{j}\log^{\ddd}n)$. Thus the change to $dev_j(h,m)$ from changes of the second type, summed over all $j\geq 2$, is bounded by the sum $\sum_{j\geq 2}(\log^{-\eee}n)^{j-1}\left(\log^{-\eee}n+j+\sqrt{j}\log^{\ddd}n\right)=o(1)$, as desired.

Putting the pieces together, the closeness of $h$ and $\bar{u}$ implies by the above Lipschitz argument that changing the distribution between $h$ and $\bar{u}$, under the  fixed median $m_{\bar{u},j}$ does not increase $dev(\cdot,\cdot)$ too much: $\sum_{j<\log^{\ccc}n}dev_{j}(h,m_{\bar{u},j})\leq o(1)+\sum_{j<\log^{\ccc}n}dev_{j}(\bar{u},m_{\bar{u},j}).$ Further, since $m_{\bar{u},j}$ minimizes this last expression, the right hand side can only increase if we replace $dev_{j}(\bar{u},m_{\bar{u},j})$ by $dev_{j}(\bar{u},m_{h,j})$ in this last inequality. Finally, a second application of the same Lipschitz property implies $\sum_{j<\log^{\ccc}n}dev_{j}(\bar{u},m_{h,j})\leq o(1)+\sum_{j<\log^{\ccc}n}dev_{j}(h,m_{h,j})$. Combining these three inequalities yields the bound of the lemma, $\sum_{j<\log^{\ccc}n}dev_{j}(h,m_{\bar{u},j})\leq o(1)+\sum_{j<\log^{\ccc}n}dev_{j}(h,m_{h,j})$, as desired.
\end{proof}

The following lemma characterizes the effect of ``fattening" in the second step of Algorithm~\ref{def:alg2}, showing that this slight modification to the histogram keeps the resulting medians small enough that me may apply the following Lemma~\ref{lem:bucket}.
\begin{lemma}\label{lem:median}
 For sufficiently large $n$, given a fattened distribution $\bar{\mu}$, for any $j<\log^{\ccc} n$, the median $m_{\bar{\mu},j,n}$ is at most $\frac{2}{n}\log^{\ccc} n$.
\end{lemma}
\begin{proof}
  Recall that $m_{\bar{\mu},j,n}$ is defined as the median of the multiset of probabilities of $\bar{u}$ after each probability $x$ has been weighted by $poi(xn,j)$. For $x\geq \frac{2}{n}\log^{\ccc} n$ and $j<\log^{\ccc} n$, these weights will each be $n^{-\Omega(1)}$ small by Poisson tail bounds; and because of the fattening, the elements added at probability $\frac{j}{n}$ will contribute inverse polylogarithmic weight. Since the median must have at most half the weight to its left, the median cannot be as large as our bound $\frac{2}{n}\log^{\ccc} n$, as desired.
\end{proof}

Given the above bound on the size of medians for small $j$, the following lemma shows that our $dev(\cdot,\cdot)$ estimates accurately capture the performance of these medians on any faithful set of samples.

\begin{lemma}\label{lem:bucket}
  Given a histogram $h$, a number of samples $n$, and for each fingerprint entry $j<\log^{\ccc} n$ a probability $m_j<\frac{2}{n}\log^{\ccc} n$ to which we attribute each domain element that shows up $j$ times in the sample, then for any faithful set of samples from $h$, the total error made for all $j<\log^{\ccc}n$ is within $o(1)$ of $\sum_{j<\log^{\ccc}n}dev_{j,n}(h,m_j)$.
\end{lemma}
\begin{proof}
Recalling the ``buckets" from Definition~\ref{def:buckets}, consider for arbitrary integer $k$, those elements of $h$ in bucket $k$, which we denote $h_k$---namely, those probabilities of $h$ lying in the interval $(\frac{k}{n\log^{\ccc} n},\frac{k+1}{n\log^{\ccc} n}]$, where by the first condition of ``faithful'', none of these probabilities are above $\frac{2}{n}\log^{\ccc}n$ for large enough $n$. Further, let $S_{j,k}$ be the multiset of probabilities of those domain elements from bucket $k$ of $h$ that each get seen exactly $j$ times in the sample. The total error of our estimate $m_j$ on bucket $k$ is thus $\sum_{x\in S_{j,k}} |m_j-x|$, which since buckets have width $1/(n\log^{\ccc} n)$, is within $|S_{j,k}|/(n\log^{\ccc} n)$ of $|S_{j,k}|\cdot|m_j-k/(n\log^{\ccc} n)|$, where we have approximated each $x$ by the left endpoint of the bucket containing $x$. By the second condition of ``faithful", $S_{j,k}$ is within $n^{0.6}$ of its expectation, $B_{poi}(j,k)$, and since by assumption $m_j<\frac{2}{n}\log^{\ccc} n$, we have that our previous error bound $|S_{j,k}|\cdot|m_j-k/(n\log^{\ccc} n)|$ is within $\frac{2}{n^{0.4}}\log^{\ccc} n$ of $B_{poi}(j,k)\cdot|m_j-k/(n\log^{\ccc} n)|$. We rewrite this final expression via the definition of $B_{poi}$ as $\sum_{x:h_k(x)\neq 0} |m-k/(n\log^{\ccc} n)|h(x)poi(nx,j)$. We compare this final expression to the portion of the deviation $dev_{j,n}(h,m_j)$ that comes from bucket $k$, namely $\sum_{x:h_k(x)\neq 0} |m_j-x|h(x)poi(nx,j)$, where since $\sum_{x:h_k(x)\neq 0} |m_j-x|h(x)poi(nx,j)=B_{poi}(j,k)$ and $x$ is within $1/(n\log^{\ccc}n)$ of $k/(n\log^{\ccc}n)$, the difference between them is clearly bounded by $B_{poi}(j,k)/(n\log^{\ccc} n)$. Using the triangle inequality to add up the three error terms we have accrued yields that our estimate for the $\ell_1$ error we make for elements seen $j$ times from bucket $k$ is accurate to within \[|S_{j,k}|/(n\log^{\ccc} n)+\frac{2}{n^{0.4}}\log^{\ccc} n+B_{poi}(j,k)/(n\log^{\ccc} n).\]

We sum this error bound over all $2\log^{2\cdot \ccc}n$ buckets $k$ and all indices $j<\log^{\ccc}n$. The middle term $\frac{2}{n^{0.4}}\log^{\ccc} n$ clearly sums up to $o(1)$ over all $j,k$ pairs. Further, since $S_{j,k}$ is within $n^{0.6}$ of $B_{poi}(j,k)$ by the definition of faithful, the sum of the first term is within $o(1)$ of the sum of the third term and it remains only to analyze the third term involving $B_{poi}(j,k)$. From its definition, $\sum_{j,k} B_{poi}(j,k)$ is the expected number of distinct items seen, when making $Poi(n)$ draws from the distribution, throwing out those elements which violate the $j$ and $k$ constraints; hence this sum over all $j,k$ pairs is at most $n$, bounding the total error of our ``$dev$" estimates by $O(1/\log^{\ccc}n)$, as desired.
\end{proof}

\subsection{Proof of Theorem~\ref{thm:main}}
We now assemble the pieces and prove Theorem~\ref{thm:main}.

\begin{proof}[Proof of Theorem~\ref{thm:main}]
Consider the output of Algorithm~\ref{def:alg1} as run in the first step of Algorithm~\ref{def:alg2}. Corollary~\ref{prop:hist} outlines two cases: with $o(1)$ probability the closeness property outlined in the proposition fails to hold, and in this case, Algorithm~\ref{def:alg2} may output a distribution up to $\ell_1$ distance 2 from the true distribution;  because this is a low-probability event, this contributes $2\cdot o(1)=o(1)$ to the expected error.  Otherwise, $u$ is close to $h$, and the fattened version $\bar{u}$ is similarly close, which lets us apply Lemma~\ref{lem:lip} to conclude that $\sum_{j<\log^{\ccc}n}dev_{j,n}(h,m_{\bar{u},j,n})\leq o(1)+ \sum_{j<\log^{\ccc}n} dev_{j,n}(h,m_{h,j,n})$. Corollary~\ref{cor:opt-est} says that $\sum_{j<\log^{\ccc}n} dev_{j,n}(h,m_{h,j,n})$ essentially lowerbounds the optimal error $opt(h,n)$, which we combine with the previous bound to yield $$\sum_{j<\log^{\ccc}n}dev_{j,n}(h,m_{\bar{u},j,n})\leq opt(h,n)+o(1).$$

Lemma~\ref{lem:faithful2} guarantees that the samples will be faithful except with $o(1)$ probability, which, as above, means that even if these unfaithful cases contribute the maximum possible distance 2 to the $\ell_1$ error, the expected contribution from these cases is still $o(1)$, and thus we will assume a faithful set of samples below. Lemmas~\ref{lem:median} and~\ref{lem:bucket} imply that for any faithful sample, the error made by Algorithm~\ref{def:alg2} on attributing those elements seen fewer than $\log^{\ccc}n$ times is within $o(1)$ of $\sum_{j<\log^{\ccc}n}dev_{j,n}(h,m_{\bar{u},j,n})$, and hence at most $o(1)$ worse than $opt(h,n)$.

Condition 1 of the definition of faithful (Definition~\ref{def:faithful2}) implies that all of the elements seen at least $\log^{\ccc}n$ times originally had probability at least $\frac{1}{n}(\log^{\ccc}n-\log^{1.75}n)$ and that the relative error between the number of  times each of these elements is seen and its expectation is thus at most $\log^{-1/4}n$. Thus using the empirical estimate on those elements appearing at least $\log^{\ccc}n$ times---as Algorithm~\ref{def:alg2} does---contributes $O(\log^{-1/4}n)$ total error on these elements. Thus all the sources of error add up to at most $o(1)$ worse than $opt(h,n)$ in expectation, yielding the theorem.
\end{proof}

\section{Proof of Fact~\ref{fact:em2l1}}~\label{ap:fact1}

For convenience, we restate Fact~\ref{fact:em2l1}:\\

\medskip

\noindent \textbf{Fact~\ref{fact:em2l1}} \emph{
Given two distributions $p_1,p_2$ satisfying $R_{\tau}(p_1,p_2) \le \eps,$  there exists a relabeling $\pi$ of the support of $p_2$ such that $$\sum_{i}\left| \max(p_1(i),\tau)-\max(p_2(\pi(i)),\tau)\right| \le 2\eps.$$}
\medskip

\begin{proof}[Proof of Fact~\ref{fact:em2l1}]
We relate relative earthmover distance to the minimum $L_1$ distance between relabled histograms, with a proof that extends to the case where both distances are defined above a cutoff threshold $\tau$. The main idea is to point out that ``minimum rearranged" $L_1$ distance can be expressed in a very similar form to earthmover distance. Given two histograms $h_1,h_2$, the minimum $L_1$ distance between any labelings of $h_1$ and $h_2$ is clearly the $L_1$ distance between the labelings where we match up elements of the two histograms in sorted order. Further, this is seen to equal the (regular, not relative) earthmover distance between the histograms $h_1$ and $h_2$, where we consider there to be $h_1(x)$ ``histogram mass'' at each location $x$ (instead of $h_1(x)\cdot x$ ``probability mass" as we did for relative earthmover distance), and place extra histogram entries at 0 as needed so the two histograms have the same total mass.

Given this correspondence, consider an optimal \emph{relative} earthmoving scheme between $h_1$ and $h_2$, and in particular, consider an arbitrary component of this scheme, where some  probability mass $\alpha$ gets moved from some location $x$ in one of the distributions to some location $y$ in the other, at cost $\alpha\log \frac{\max(x,\tau)}{\max(y,\tau)}$, and suppose without loss of generality that $x\geq y$.

We now reinterpret this move in the $L_1$ sense, translating from moving probability mass to moving histogram mass. In the non-relative earthmover problem, $\alpha$ probability mass at location $x$ corresponds to $\frac{\alpha}{x}$ ``histogram mass" at $x$, which we then move to $y$ at cost $(\max(x,\tau)-\max(y,\tau))\frac{\alpha}{x}$; however, to simulate the relative earthmover scheme, we need the full $\frac{\alpha}{y}$ mass to appear at $y$, so we move the remaining $\frac{\alpha}{y}-\frac{\alpha}{x}$ mass up from 0, at cost $(\frac{\alpha}{y}-\frac{\alpha}{x})(\max(y,\tau)-\tau)$.

To relate these 3 costs (the original relative earthmover cost, and the two components of the non-relative histogram earthmover cost), we note that if both $x$ and $y$ are less than or equal to $\tau$ then all 3 costs are 0. Otherwise, if $x,y>\tau$ then the first component of the histogram cost equals $(1-\frac{y}{x})\alpha$ and the second is bounded by this, as $(\frac{\alpha}{y}-\frac{\alpha}{x})(\max(y,\tau)-\tau)< (\frac{\alpha}{y}-\frac{\alpha}{x})y=(1-\frac{y}{x})\alpha$. Further, for the case under consideration where $\tau<y\leq x$, we have $(1-\frac{y}{x})\alpha\leq \alpha\log\frac{x}{y}$, which equals the relative earthmover cost. Thus the histogram cost in this case is at most twice the relative earthmover cost.

In the remaining case, $y\leq\tau< x$, and the second component of the histogram cost equals 0 because $\max(y,\tau)-\tau=0$. The first component simplifies as $(\max(x,\tau)-\max(y,\tau))\frac{\alpha}{x}=(x-\tau)\frac{\alpha}{x}=(1-\frac{\tau}{x})\alpha \leq \alpha\log\frac{x}{\tau}$, where this last expression is the relative earthmover cost. Thus in all cases, the histogram cost is at most twice the relative earthmoving cost.

Since the histogram cost was one particular ``histogram moving scheme", and as we argued above, the ``minimum permuted $L_1$ distance" is the minimum over all such schemes, we conclude that this $L_1$ distance is at most twice the relative earthmover distance, as desired.

\end{proof}

\section{Proof of Lemma~\ref{lem:unique}}\label{appendix:lem:unique}

For convenience, we restate the lemma:
\medskip

\noindent \textbf{Lemma~\ref{lem:unique}} \emph{
Given two (possibly generalized) histograms $g,h$, a number of samples $k$, and a threshold $\tau\in(0,1]$,
\[\left|\sum_{x:g(x)\neq 0} (1-(1-x)^k)\cdot g(x)-\sum_{x:h(x)\neq 0} (1-(1-x)^k)\cdot h(x)\right|\leq (0.3(k-1)+1)R_\tau(g,h)+\tau \frac{k}{2}\]
}
\medskip
\begin{proof}
We prove the inequality by considering each step of an earthmoving scheme that transforms $g$ to $h$, and show that if in one step $m$ probability mass is moved, at $\tau$-truncated relative earthmover cost $r$, then the sum $\sum_{x:g(x)\neq 0} (1-(1-x)^k)\cdot g(x)$ changes by at most $(1+0.3(k-1))\cdot r+mk\tau$, meaning that an entire earthmoving scheme to transform $g$ into $h$ with total cost $R_\tau(g,h)$ and total mass at most 1 changes the $g$ term on the left hand side into the $h$ term on the left hand side by changing it at most $(1+0.3(k-1))\cdot R_\tau(g,h)+k\tau$.

To prove this we first analyze the region of probability below $\tau$. By the definition of a histogram, $m$ units of probability mass at probability $x$ corresponds to a histogram entry $h(x)=\frac{m}{x}$, and binomial bounds yield $\frac{m}{x}(1-(1-x)^k)\in [km(1-x\frac{k-1}{2}), k m]$, which means that when an earthmoving scheme moves $m$ mass in the range $x\in (0,\tau]$, the expression $\frac{m}{x}(1-(1-x)^k)$ changes by at most $k m\frac{k-1}{2}\tau$. Thus, summed over the entire earthmoving scheme, where the mass moved sums to at most 1, the change in $\sum_{x:g(x)\neq 0} (1-(1-x)^k)\cdot g(x)$ from changes below probability $\tau$ is at most $k\frac{k-1}{2}\tau$.

To bound the remaining term, changes in $\sum_{x:g(x)\neq 0} (1-(1-x)^k)\cdot g(x)$ from changes in probability above $\tau$ in the earthmoving scheme, we note that to move probability mass $m$ from probability value $x$ to $y$ costs $m|\log x-\log y|$ in the earthmoving scheme, and changes the sum by $$\left| \frac{m}{x} \left( 1 - (1-x)^k\right) - \frac{m}{y}\left(1-(1-y)^k\right)  \right|.$$ We bound the ratio of these last two expressions by $1+0.3(k-1)$, in order to bound the total contribution of the portion of the earthmoving scheme above probability $\tau$ by $(1+0.3(k-1))R_\tau(g,h)$, yielding the desired overall bound.

We thus seek to bound the maximum change in $\frac{1}{x}\left(1-(1-x)^k\right)$ relative to the change in $\log x$ as $x$ changes, namely the maximum ratio of their derivatives, where we add a negative sign since $\frac{1}{x}\left(1-(1-x)^k\right)$ is a decreasing function. Since $\frac{d}{dx}\log x=1/x$, the ratio of derivatives is \begin{equation}\label{eq:deriv}-x\frac{d}{dx}\frac{\left(1-(1-x)^k\right)}{x}=\frac{1-(1-x)^{k-1}((k-1)x+1)}{x}\end{equation}

Consider the approximation $(1-x)^{k-1}\approx e^{-x(k-1)}$. Taking logarithms of both sides, and using the fact that, for $x\leq \frac{1}{2}$, we have $\log 1-x\geq -x-x^2$, we have that for $x\leq\frac{1}{2}$ the inequality $(k-1)\log (1-x)\geq -(k-1)(x+x^2)$; exponentiating yields $(1-x)^{k-1}\geq e^{-x(k-1)}\cdot e^{-x^2(k-1)}\geq e^{-x(k-1)}(1-x^2(k-1))$. 

Thus for $x\leq\frac{1}{2}$ the ratio of derivatives is bounded as \begin{align*}-x\frac{d}{dx}\frac{\left(1-(1-x)^k\right)}{x}\leq & \frac{1-(e^{-x(k-1)}(1-x^2(k-1)))((k-1)x+1)}{x} \\
=&\frac{1-e^{-x(k-1)}((k-1)x+1)}{x} + \frac{e^{-x(k-1)}x^2(k-1)((k-1)x+1)}{x}
\end{align*}

The first term of the right hand side, after dividing by $k-1$, can be reexpressed in terms of $y=x(k-1)$ as $\frac{1-e^{-y}(y+1)}{y}$, which has a global maximum less then $0.3$; the second term in the right hand side, after the same variable substitution, equals $e^{-y}y(y+1)$, which has a global maximum less than 1. Thus, for $x\leq\frac{1}{2}$, the absolute value of the ratio of derivatives is bounded as $0.3(k-1)+1$. For $x\geq\frac{1}{2}$, the right hand side of Equation~\ref{eq:deriv} is $\frac{1}{x}$ minus some positive quantity, and is hence at most $2$.  Since $0.3(k-1)+1 \ge 2$ for any $k \ge 5,$ all that remains is to checking the $k=2, 3, 4$ cases where $0.3(k-1)+1<2$ by hand to confirms that $0.3(k-1)+1$ is in fact a global bound.
\end{proof}

\section{Proof of Theorem~\ref{thm:h2}}\label{ap:thm2}
In this section, we prove Theorem~\ref{thm:h2}, characterizing the performance of Algorithm~\ref{def:alg1} which recovers an accurate approximation of the histogram of the true distribution.  For convenience, we restate Theorem~\ref{thm:h2}:\\

\medskip

\noindent \textbf{Fact~\ref{fact:em2l1}} \emph{There exists an absolute constant $c$ such that for sufficiently large $n$ and any $w \in [1,\log n],$ given $n$ independent draws from a distribution $p$ with histogram $h$, with probability $1-e^{-n^{\Omega(1)}}$ the generalized histogram $h_{LP}$ returned by Algorithm~\ref{def:alg1} satisfies $$R_{\frac{w}{n \log n}}(h,h_{LP}) \le \frac{c}{\sqrt{w}}.$$}

\medskip

The proof decomposes into three parts.  In Appendix~\ref{sec:prob} we compartmentalize the probabilistic portion of the proof by defining a set of conditions that are satisfied with high probability, such that if the samples in question satisfy the properties, then the algorithm will succeed.  This section is analogous to the definition of a ``faithful'' set of samples of Definition~\ref{def:faithful2}, and
we re-use the terminology of ``faithful''.   In Appendix~\ref{sec:goodFeas} we show that, provided the samples in question are ``faithful'', there exists a feasible solution to the linear program defined in
Algorithm~\ref{def:alg1}, which 1) has small objective function value, and 2) is very close to the true histogram from which the samples were drawn, in terms of $\tau$-truncated relative earthmover
distance---for an appropriate choice of $\tau$.  In Appendix~\ref{sec:ChebEM} we show that if two feasible solutions to the linear program defined in Algorithm~\ref{def:alg1} both have small objective function value, then they are close in $tau$-truncated relative earthmover distance.  The key tool here is a Chebyshev polynomial earthmover scheme.   Finally, in Appendix~\ref{sec:fp2}, we put together the above pieces to prove Theorem~\ref{thm:h2}: given the existence of a feasible point that has low-objective function value that is close to the true histogram, and the fact that any two solutions that both have low objective function value must be close to each other, it follows that the solution to the linear program that is found in Algorithm~\ref{def:alg1} must be close to the true histogram.

\subsection{Compartmentalizing the Probabilistic Portion}\label{sec:prob}

The following condition defines what it means for a set of samples drawn from a distribution to be ``faithful'' with respect to positive constants $\mathcal{B},\mathcal{D} \in (0,1)$:

\begin{definition}\label{def:faithful_b}
A set of $n$ samples with fingerprint $\FF$, drawn from a distribution $p$ with histogram $h$, is said to be \emph{faithful} with respect to positive constants $\mathcal{B},\mathcal{D} \in (0,1)$ if the following conditions hold:
\begin{itemize}
    \item{For all $i$, $$\left| \FF_i - \sum_{x:h(x) \neq 0} h(x) \cdot poi(nx,i) \right| \le \max\left(  \FF_i^{\frac{1}{2}+\constD},n^{\constB(\frac{1}{2}+\constD)}\right).$$ }
        \item{For all domain elements $i,$ letting $p(i)$ denote the true probability of $i$, the number of times $i$ occurs in the samples from $p$ differs from $n\cdot p(i)$ by at most $$\max \left( \left( n \cdot p(i)\right)^{\frac{1}{2}+\constD}, n^{\constB(\frac{1}{2}+\constD)} \right).$$}
        \item{The ``large'' portion of the fingerprint $\FF$ does not contain too many more samples than expected:  Specifically, $$\sum_{i >n^\constB + 2n^\constC}\FF_i \le n^{1/2+\constD}+n\sum_{x \le \frac{n^\constB+n^\constC}{n}: h(x)> 0} x \cdot h(x).$$}
\end{itemize}
\end{definition}

The following proposition is proven via the standard ``Poissonization'' technique and Chernoff bounds.

\begin{proposition}\label{lemma:faithfulwhp_b}
For any constants $\mathcal{B},\mathcal{D} \in (0,1)$, there is a constant $\alpha>0$ and integer $n_0$ such that for any $n \ge n_0$, a set of $n$ samples consisting of independent draws from a distribution is ``faithful'' with respect to $\mathcal{B},\mathcal{D}$ with probability at least $1-e^{-n^{\alpha}}.$
\end{proposition}
\begin{proof}
We first analyze the case of a $Poi(n)$-sized sample drawn from a distribution with histogram $h$.   Thus $$\E[\FF_i] =  \sum_{x:h(x)\neq 0} h(x) poi(nx,i).$$  Additionally, the number of times each domain element occurs is independent of the number of times the other domain elements occur, and thus each fingerprint entry $\FF_i$ is the sum of independent random $0/1$ variables, representing whether each domain element occurred exactly $i$ times in the samples (i.e. contributing $1$ towards $\FF_i$).  By independence, Chernoff bounds apply.

We split the analysis into two cases, according to whether $\E[\FF_i] \ge n^\constB.$   In the case that $\E[\FF_i] < n^\constB,$ we leverage the basic Chernoff bound that if $X$ is the sum of independent $0/1$ random variables with $\E[X] \le S,$ then for any $\delta \in (0,1),$ $$\Pr[|X-\E[X]| \ge \delta S] \le 2e^{- \delta^2 S/3}.$$  Applied to our present setting where $\FF_i$ is a sum of independent $0/1$ random variables, provided $\E[\FF_i] < n^\constB,$ we have: $$\Pr\left[\left| \FF_i - \E[\FF_i] \right| \ge (n^\constB)^{\frac{1}{2}+\constD}\right] \le 2 e^{-\left(\frac{1}{(n^\constB)^{1/2-\constD}}\right)^2 \frac{n^\constB}{3}} = 2  e^{-n^{2\constB \constD}/3}.$$

In the case that $\E[\FF_i] \ge n^\constB,$ the same Chernoff bound yields $$\Pr\left[\left| \FF_i - \E[\FF_i] \right| \ge \E[\FF_i]^{\frac{1}{2}+\constD}\right] \le  2 e^{-\left(\frac{1}{\E[\FF_i]^{1/2-\constD}}\right)^2 \frac{\E[\FF_i]}{3}} = 2 e^{-\left(\E[\FF_i]^{2\constD}\right)/3} \le 2  e^{-n^{2\constB \constD}/3}.$$
A union bound over the first $n$ fingerprints shows that the probability that given a set of samples (consisting of $Poi(n)$ draws), the probability that any of the fingerprint entries violate the first condition of \emph{faithful} is at most $n\cdot 2 e^{-\frac{n^{2 \constB \constD}}{3}} \le e^{-{n^{\Omega(1)}}}$ as desired.

For the second condition of ``faithful'', in analogy with the above argument, for any $\lambda \le S,$ and $\delta \in (0,1),$ $$\Pr[|Poi(\lambda) -\lambda| > \delta S] \le 2 e^{-\delta^2 S/3}.$$   Hence for $x=n\cdot p(i) \ge n^\constB,$ the probability that the number of occurrences of domain element $i$ differs from its expectation of $n \cdot p(i)$ by at least $(n \cdot p(i))^{\frac{1}{2}+\constD}$ is bounded by $2e^{-(n\cdot p(i))^{2\constD}/3} \le e^{-n^{\Omega(1)}}.$  Similarly, in the case that $x=n\cdot p(i) < n^\constB,$ $$\Pr[|Poi(x)-x| > n^{\constB(\frac{1}{2}+\constD)}] \le e^{-n^{\Omega(1)}}.$$

For the third condition, by the Poisson tail bounds of the previous paragraph, the total aggregate number of occurrences of all elements with probability greater than $\frac{n^\constB + n^\constC}{n}$ will differ from its expectation by at most $n^{1/2+\constD}$, with probability $1-e^{-n^{\Omega(1)}}$.  Additionally, by the first condition of ``faithful'', with probability $1-e^{-n^{\Omega(1)}}$ no domain element $i$ with $p(i) < \frac{n^\constB + n^\constC}{n}$ will appear more than $n^\constB +2n^\constC$.   Hence with probability $1-e^{-n^{\Omega(1)}}$ all elements that contribute to the sum $\sum_{i > n^\constB+2n^\constC} \FF_i$ will have probability greater than $\frac{n^\constB + n^\constC}{n}.$   The third condition then follows by a union bound over these two $e^{-n^{\Omega(1)}}$ failure probabilities.

Thus we have shown that provided we are considering a sample size of $Poi(n),$ the probability that the conditions hold is at least $1-e^{-n^{\Omega(1)}}.$  To conclude, note that $\Pr[Poi(n) = n] > \frac{1}{3 \sqrt{n}},$ and hence the probability that the conditions do not hold for a set of exactly $n$ samples (namely, the probability that they do not hold for a set of $Poi(n)$ samples, conditioned on the sample size being exactly $n$), is at most a factor of $3\sqrt{n}$ larger, and hence this probability of failure is still $e^{-n^{\Omega(1)}},$ as desired.
\end{proof}

\subsection{Existence of a Good Feasible Point}\label{sec:goodFeas}
\begin{proposition}\label{prop:goodfeasible}
Provided $\FF$ is a ``faithful'' fingerprint derived from a distribution with histogram $h$, there exists a feasible point, $(v_x)$, for the linear program of Algorithm~\ref{def:alg1} with objective function value at most $O(n^{\frac{1}{2}+\constB+\constD})$ such that for any $\tau > 1/n^{3/2},$ the $\tau$-truncated relative earthmover distance between the generalized histogram corresponding to $(v_x)$ with the empirical fingerprint $\FF_{i>n^\constB+2n^\constC}$ appended, and the true histogram, $h$, is bounded by $O\left(\max(n^{-\constB(\frac{1}{2}-\constD)},n^{-(\constB-\constC)}\right),$ where the big O hides an absolute constant.
\end{proposition}
\begin{proof}
Let $(v_x)$ be defined as follows:  initialize $(v_x)$ to be identically zero.  For each $y\le \frac{n^{\constB}+n^{\constC}}{n}$ s.t. $h(y)>0,$ increment $v_{x}$ by $h(y)\frac{y}{x}$, where $x = \min\{x \in X : x \ge y\}$.  Finally, define $$m \defeq 1- \left(\sum_{i > n^\constB+2n^{\constC}} \frac{i}{n}\FF_i + \sum_{x \in X} x \cdot v_x \right).$$  If $m > 0,$ increment $v_x$ by $m/x$ for $x = \frac{n^{\constB}+n^{\constC}}{n}.$  If $m<0,$ then arbitrarily reduce $v_x$ until a total of $m$ units of mass have been removed.

We first argue that the $\tau$-truncated relative earthmover distance is small, and then will argue about the objective function value.   Let $h'$ denote the histogram obtained by appending the empirical fingerprint $\FF_{i>n^\constB+2n^\constC}$ to $(v_x).$   We construct an earthmoving scheme between $h$ and $h'$ as follows:   1) for all $y\le \frac{n^{\constB}+n^{\constC}}{n}$ s.t. $h(y)>0,$ we move $h(y)\cdot y$ mass to location $x = \min\{x \in X : x \ge y\};$ 2) for each domain element $i$ that occurs more than $n^\constB+2n^\constC$ times, we move $p(i)$ mass from location $p(i)$ to $\frac{X_i}{n}$ where $X_i$ denotes the number of occurrences of the $i$th domain element; 3) finally, whatever discrepancy remains between $h$ and $h'$ after the first two earthmoving phases, we move to probability $\frac{n^{\constB}}{n}$.   Clearly this is an earthmoving scheme.  For $\tau \ge 1/n^{3/2},$ the $\tau$-truncated relative earthmover cost of the first phase is trivially at most $\log\frac{1/n^{3/2}+1/n^2}{1/n^{3/2}} = O(1/\sqrt{n})$. By the second condition of ``faithful'', the relative earthmover cost of the second phase of the scheme is bounded by $\log(\frac{n^{\constB}-n^{\constB(1/2+\constD)}}{n^{\constB}}) = O(n^{-\constB(\frac{1}{2}-\constD)}).$     To bound the cost of the third phase, note that the first phase equates the two histograms below probability $n^{\constB}{n}.$  By the second condition of ``faithful'', after the second phase , there is at most $O(n^{-\constB(\frac{1}{2}-\constD)})$ unmatched probability caused by the discrepancy between $\frac{X_i}{n}$ and $p(i)$ for elements observed at least $n^\constB+2n^\constC$ times.  Hence after this $O(n^{-\constB(\frac{1}{2}-\constD)})$ discrepancy is moved to probability $\frac{n^{\constB}}{n}$, the entirety of the remaining discrepancy lies in the probability range $[\frac{n^{\constB}}{n},c],$ where $c$ is an upper bound on the true probability of an element that does not appear at least $n^\constB + 2n^\constC$ times; from the second condition of ``faithful'', $c \le \frac{n^\constB + 4n^\constC}{n}$, and hence the total $\tau$-truncated relative earthmover distance is at most $O\left(\max(n^{-\constB(\frac{1}{2}-\constD)},n^{-(\constB-\constC)}\right),$ as desired.

To complete the proof of the proposition, note that by construction, $(v_x)$ is a feasible point for the linear program.  To see that the objective function is as claimed, note that $|\frac{d}{dx}\poi(nx,i)| \le n$, and since we are rounding the true histogram to probabilities that are multiples of $1/n^2$, each ``fingerprint expectation'', $\sum_{x \in X} \poi(nx,i)\cdot v_x$ differs from $\sum_{x: h(x)\neq 0} \poi(nx,i)\cdot h(x)$ by at most $1/\sqrt{n}.$  Together with the first condition of ``faithful'' which implies that each of the observed fingerprints $\FF_i$ satisfies $|\FF_i -  \sum_{x: h(x)\neq 0} \poi(nx,i)\cdot h(x)| \le n^{\frac{1}{2}+\constD},$  we conclude that the total objective function value is at most $n^{\constB}(n^{\frac{1}{2}+\constD}+1/\sqrt{n}) = O(n^{\frac{1}{2}+\constB+\constD}).$
\end{proof}

\subsection{The Chebyshev Bump Earthmoving Scheme}\label{sec:ChebEM}
\begin{proposition}
Given a ``faithful'' fingerprint $\FF_i$, then any pair of solutions $v_x, v'_x$ to the linear program of Algorithm~\ref{def:alg1} that both have objective function values at most $O(n^{\frac{1}{2}+\constB+\constD})$  satisfy the following:  for any $w \in [1,\log n],$ their $\frac{w}{n \log n}$-truncated relative earthmover distance $R_{w/n \log n}[v_x,v'_x] \le O(1/\sqrt{w}).$
\end{proposition}

The proof of the above proposition relies on an explicit earthmover scheme that leverages a Chebyshev polynomial construction.  The two key properties of the scheme are 1) the truncated relative earthmover cost of the scheme is small, and 2) given two histograms that have similar expected fingerprints (i.e. for all $i \le n^\constB$, $\sum_x v_x \poi(nx,i) \approx \sum_x v'_x \poi(nx,i),$) the results of applying the scheme to the pair of histograms will result in histograms that are very close to each other in truncated relative earthmover distance.    We outline the construction and key propositions below.

\begin{definition}
  For a given $n$, a $\beta$-\emph{bump earthmoving scheme} is defined by a sequence of positive real numbers $\{c_i\}$, the \emph{bump centers}, and a sequence of functions $\{f_i\}:(0,1]\rightarrow \mathbb{R}$ such that $\sum_{i=0}^\infty f_i(x)=1$ for each $x$, and each function $f_i$ may be expressed as a linear combination of Poisson functions, $f_i(x)=\sum_{j=0}^\infty a_{ij} poi(nx,j)$, such that $\sum_{j=0}^\infty |a_{ij}|\leq\beta$.

  Given a generalized histogram $h$, the scheme works as follows: for each $x$ such that $h(x)\neq 0$, and each integer $i\ge0$, move $x h(x)\cdot f_i(x)$ units of probability mass from $x$ to $c_i$. We denote the histogram resulting from this scheme by $(c,f)(h)$.
\end{definition}

\begin{definition}
  A bump earthmoving scheme $(c,f)$ is $[\epsilon,\tau]$-\emph{good} if for any generalized histogram $h$ the $\tau$-truncated relative earthmover distance between $h$ and $(c,f)(h)$ is at most $\epsilon$.
\end{definition}

Below we define the Chebyshev bumps to be a ``third order'' trigonometric construction:

\begin{definition}\label{def:thin-bumps_b}
The \emph{Chebyshev bumps} are defined in terms of $n$ as follows. Let $s=0.2 \log n$.
Define $g_1(y)=\sum_{j=-s}^{s-1} \cos(j y)$. Define $$g_2(y)=\frac{1}{16 s } \left(g_1(y-\frac{3\pi}{2s})+3g_1(y-\frac{\pi}{2s})+3g_1(y+\frac{\pi}{2s})+g_1(y+\frac{3\pi}{2s})\right),$$
 and, for $i\in\{1,\ldots,s-1\}$ define $g_3^i(y) \defeq  g_2(y-\frac{i\pi}{s})+ g_2(y+\frac{i\pi}{s})$, and $g_3^0 = g_2(y),$ and $g_3^s = g_2(y+\pi)$. Let $t_i(x)$ be the linear combination of Chebyshev polynomials so that $t_i(\cos(y))=g_3^i(y)$. We thus define $s+1$ functions, the ``skinny bumps", to be $B_i(x)=t_i(1-\frac{xn}{2s})\sum_{j=0}^{s-1}poi(xn,j)$, for $i\in\{0,\ldots,s\}$. That is, $B_i(x)$ is related to $g_3^i(y)$ by the coordinate transformation $x=\frac{2s}{n}(1-\cos(y))$, and scaling by $\sum_{j=0}^{s-1}poi(xn,j)$.
 \end{definition}

\begin{definition}\label{definition:bumps-construction_b}
The \emph{Chebyshev earthmoving scheme} is defined in terms of $n$ as follows:  as in Definition~\ref{def:thin-bumps_b}, let $s=0.2 \log n$. For $i\geq s+1$, define  the $i$th bump function $f_{i}(x)=poi(nx,i-1)$ and associated bump center $c_{i}=\frac{i-1}{n}$.  For $i \in \{0,\ldots, s\}$ let $f_i(x)=B_i(x),$ and for $i \in \{1,\ldots,s\},$  define their associated bump centers $c_i=\frac{2s}{n}(1-\cos(\frac{i\pi}{s}))$, with $c_0 = c_1$.
\end{definition}

The following proposition characterizes the key properties of the Chebyshev earthmoving scheme. Namely, that the scheme is, in fact, an earthmoving scheme, that each bump can be expressed as a low-weight linear combination of Poisson functions, and that the scheme incurs a small truncated relative earthmover cost.  

\begin{proposition}~\label{lemma:chebyshevEMS_b}
The Chebyshev earthmoving scheme of Definition~\ref{definition:bumps-construction_b}, defined in terms of $n$, has the following properties:
\begin{itemize}
\item{For any $x \ge 0$, $$\sum_{i \ge 0} f_i(x) = 1,$$ hence the Chebyshev earthmoving scheme is a valid earthmoving scheme.}
\item{Each $B_i(x)$ may be expressed as $\sum_{j=0}^\infty a_{ij} poi(nx,j)$ for $a_{ij}$ satisfying $$\sum_{j=0}^\infty |a_{ij}|\leq 2n^{0.3}.$$}
    \item{The Chebyshev earthmoving scheme is $\left[O(1/\sqrt{w}),\frac{w}{n \log n}\right]$-good, for any $w \in [1, \log n],$ where the $O$ notation hides an absolute constant factor.}
    \end{itemize}
\end{proposition}

The proof of the first two bullets of the proposition closely follow the arguments in~\cite{stocVV}.  For the final bullet point, the intuition of the proof is the following:  the $i$th bump $B_i$, with center $c_i = \frac{2s}{n}\left(1-\cos(i \pi/s)\right) \approx i^2 \frac{2}{ns}$ has a width of $O(\frac{i}{ns}),$ and $B_i(x)$ decays rapidly (as the fourth power) away from its center, $c_i$.  Specifically, $B_i(c_i \pm \frac{\alpha i }{ns}) \le O(1/\alpha^4).$   Hence, at worst, the cost of the earthmoving scheme will be dominated by the cost of moving the mass around the smallest $c_i$ that exceeds the truncation parameter $w/n \log n$.  Such a bump will have width $O(\frac{\sqrt{w}}{ns}) = O(\frac{\sqrt{w}}{n \log n}),$ which will incur a per-unit mass relative earthmover cost of $O(\sqrt{1/w})$.

For completeness, we give a complete proof of Proposition~\ref{lemma:chebyshevEMS_b}, with the three parts split into distinct lemmas:

\begin{lemma}\label{lemma:g2sum1}
For any $x$ $$\sum_{i=-s+1}^{s} g_2(x+\frac{\pi i}{s}) = 1, \text{ and } \sum_{i=0}^{\infty} f_i(x) = 1.$$
\end{lemma}
\begin{proof}
$g_2(y)$ is a linear combination of cosines at integer frequencies $j$, for $j=0,\ldots,s,$ shifted by $\pm \pi/2s$ and $\pm 3 \pi/s2.$   Since $\sum_{i=-s+1}^{s} g_2(x+\frac{\pi i}{s})$ sums these cosines over all possible multiples of $\pi/s$, we note that all but the frequency 0 terms will cancel.  The $\cos( 0 y) = 1$ term will show up once in each $g_1$ term, and thus $1+3+3+1=8$ times in each $g_2$ term, and thus $8\cdot 2 s$ times in the sum in question.  Together with the normalizing factor of $16 s,$ the total sum is thus $1$, as claimed.

For the second part of the claim, \begin{eqnarray*}\sum_{i=0}^{\infty} f_i(x) & = &   \left(\sum_{j=-s+1}^s g_2(\cos^{-1}\left(\frac{xn}{2s}-1\right)+\frac{\pi j}{s}) \right)\sum_{j=0}^{s-1}poi(xn,j) + \sum_{j\ge s} poi(xn,j)\\ & = & 1 \cdot \sum_{j=0}^{s-1}poi(xn,j) + \sum_{j\ge s} poi(xn,j) = 1.\end{eqnarray*}
\end{proof}

We now show that each Chebyshev bump may be expressed as a low-weight linear combination of Poisson functions.

\begin{lemma}~\label{lemma:ChebasP}
Each $B_i(x)$ may be expressed as $\sum_{j=0}^\infty a_{ij} poi(nx,j)$ for $a_{ij}$ satisfying $$\sum_{j=0}^\infty |a_{ij}|\leq 2n^{0.3}.$$
\end{lemma}
\begin{proof}
     Consider decomposing $g_3^i(y)$ into a linear combination of $\cos(\ell y)$, for $\ell\in\{0,\ldots,s\}$. Since $\cos(-\ell y)=\cos(\ell y)$, $g_1(y)$ consists of one copy of $\cos(sy)$, two copies of $\cos(\ell y)$ for each $\ell$ between 0 and $s$, and one copy of $\cos(0y)$; $g_2(y)$ consists of ($\frac{1}{16s}$ times) 8 copies of different $g_1(y)$'s, with some shifted so as to introduce sine components, but these sine components are canceled out in the formation of $g_3^i(y)$, which is a symmetric function for each $i$. Thus since each $g_3$ contains at most two $g_2$'s, each $g_3^i(y)$ may be regarded as a linear combination $\sum_{\ell=0}^s \cos(\ell y)b_{i\ell}$ with the coefficients bounded as $|b_{i\ell}|\leq\frac{2}{s}$.

Since $t_i$ was defined so that $t_i(\cos(y))=g_3^i(y)=\sum_{\ell=0}^s \cos(\ell y)b_{i\ell}$, by the definition of Chebyshev polynomials we have $t_i(z)=\sum_{\ell=0}^s T_\ell(z)b_{i\ell}$. Thus the bumps are expressed as $$B_i(x)=\left(\sum_{\ell=0}^s T_\ell(1-\frac{xn}{2s})b_{i\ell}\right)\left(\sum_{j=0}^{s-1} poi(xn,j)\right).$$ We further express each Chebyshev polynomial via its coefficients as $T_\ell(1-\frac{xn}{2s})=\sum_{m=0}^\ell \beta_{\ell m} (1-\frac{xn}{2s})^m$ and then expand each term via binomial expansion as $(1-\frac{xn}{2s})^m=\sum_{q=0}^m (-\frac{xn}{2s})^q {m\choose q}$ to yield $$B_i(x)=\sum_{\ell=0}^s\sum_{m=0}^\ell\sum_{q=0}^m\sum_{j=0}^{s-1} \beta_{\ell m} \left(-\frac{xn}{2s}\right)^q{m\choose q}b_{i\ell}\,poi(xn,j).$$ We note that in general we can reexpress $x^q\, poi(xn,j)=x^q\frac{x^jn^je^{-{xn}}}{j!}=poi(xn,j+q)\frac{(j+q)!}{j!n^q}$, which finally lets us express $B_i$ as a linear combination of Poisson functions, for all $i \in \{0,\ldots, s\}$: $$B_i(x)=\sum_{\ell=0}^s\sum_{m=0}^\ell\sum_{q=0}^m\sum_{j=0}^{s-1} \beta_{\ell m} \left(-\frac{1}{2s}\right)^q{m\choose q}\frac{(j+q)!}{j!}b_{i\ell}\,poi(xn,j+q).$$

     It remains to bound the sum of the absolute values of the coefficients of the Poisson functions. That is, by the triangle inequality, it is sufficient to show that
     $$\sum_{\ell=0}^s\sum_{m=0}^\ell\sum_{q=0}^m\sum_{j=0}^{s-1} \left|\beta_{\ell m} \left(-\frac{1}{2s}\right)^q{m\choose q}\frac{(j+q)!}{j!}b_{i\ell}\right|\leq 2n^{0.3}$$

     We take the sum over $j$ first: the general fact that $\sum_{m=0}^\ell {m+i\choose i}={i+\ell+1\choose i+1}$ implies that $\sum_{j=0}^{s-1}\frac{(j+q)!}{j!}=\sum_{j=0}^{s-1}{j+q\choose q}q!=q!{s+q\choose q+1}=\frac{1}{q+1}\frac{(s+q)!}{(s-1)!}$, and further, since $q\leq m\leq \ell\leq s$ we have $s+q\leq 2s$ which implies that this final expression is bounded as $\frac{1}{q+1}\frac{(s+q)!}{(s-1)!}=s\frac{1}{q+1}\frac{(s+q)!}{s!}\leq s\cdot (2s)^q$. Thus we have \begin{align*}\sum_{\ell=0}^s\sum_{m=0}^\ell\sum_{q=0}^m\sum_{j=0}^{s-1} \left|\beta_{\ell m} \left(-\frac{1}{2s}\right)^q{m\choose q}\frac{(j+q)!}{j!}b_{i\ell}\right| &\leq& \sum_{\ell=0}^s\sum_{m=0}^\ell\sum_{q=0}^m \left|\beta_{\ell m} s {m\choose q} b_{i\ell}\right|\\&=&s\sum_{\ell=0}^s |b_{i\ell}|\sum_{m=0}^\ell |\beta_{\ell m}| 2^m \end{align*}

 Chebyshev polynomials have coefficients whose signs repeat in the pattern $(+,0,-,0)$, and thus we can evaluate the innermost sum exactly as $|T_\ell(2\ii)|$, for $\ii=\sqrt{-1}$. Since we bounded $|b_{i\ell}|\leq \frac{2}{s}$ above, the quantity to be bounded is now $s\sum_{\ell=0}^s\frac{2}{s}|T_\ell(2\ii)|$. Since the explicit expression for Chebyshev polynomials yields $|T_\ell(2\ii)|=\frac{1}{2}\left[(2-\sqrt{5})^\ell+(2+\sqrt{5})^\ell\right]$ and since $|2-\sqrt{5}|^\ell=(2+\sqrt{5})^{-\ell}$ we finally bound $s\sum_{\ell=0}^s\frac{2}{s}|T_\ell(2\ii)|\leq 1+\sum_{\ell=-s}^s (2+\sqrt{5})^\ell<1+\frac{2+\sqrt{5}}{2+\sqrt{5}-1}\cdot(2+\sqrt{5})^s<2\cdot (2+\sqrt{5})^s < 2\cdot k^{0.3}$,  as desired, since $s=0.2\log n$ and $\log(2+\sqrt{5})<1.5$ and $0.2\cdot 1.5=0.3$.
\end{proof}

The following lemma quantifies the ``skinnyness'' of the Chebyshev bumps, which is the main component in the proof of the quality of the scheme (the third bullet in Proposition~\ref{lemma:chebyshevEMS_b}).

\begin{lemma}\label{lemma:g3bounds}
 $|g_2(y)| \le \frac{\pi^7}{y^4 s^4}$ for $y \in [-\pi,\pi] \setminus (-3\pi/s, 3 \pi/s),$ and $|g_2(y)| \le 1/2$ everywhere.
\end{lemma}
\begin{proof}
Since $g_1(y) =  \sum_{j=-s}^{s-1} \cos{jy} = \sin(sy) \cot(y/2),$ and since $\sin(\alpha+\pi)=-\sin(\alpha),$ we have the following:
\begin{eqnarray*}
     g_2(y) & = & \frac{1}{16 s} \left(g_1(y-\frac{3\pi}{2s})+3g_1(y-\frac{\pi}{2s})+3g_1(y+\frac{\pi}{2s})+ g_1(y+\frac{3\pi}{2s})\right) \\
     & = & \frac{1}{16 s} \left(\sin(y s + \pi/2) \left( \cot(\frac{y}{2}-\frac{3 \pi}{4 s}) - 3\cot(\frac{y}{2}-\frac{ \pi}{4 s}) \right. \right. \\ & & \hspace{5cm} \left. \left. +3\cot(\frac{y}{2}+\frac{ \pi}{4 s}) - \cot(\frac{y}{2}+\frac{3 \pi}{4 s}) \right) \right).
  \end{eqnarray*}
 Note that $\left( \cot(\frac{y}{2}-\frac{3 \pi}{4 s}) - 3\cot(\frac{y}{2}-\frac{ \pi}{4 s})+3\cot(\frac{y}{2}+\frac{ \pi}{4 s}) - \cot(\frac{y}{2}+\frac{3 \pi}{4 s}) \right)$ is a discrete approximation to $(\pi/2s)^3$ times the third derivative of the cotangent function evaluated at $y/2$.  Thus it is bounded in magnitude by $(\pi/2s)^3$ times the maximum magnitude of $\frac{d^3}{dx^3}\cot(x)$ in the range $x \in [\frac{y}{2}-\frac{3\pi}{4s},\frac{y}{2}+\frac{3\pi}{4s}].$  Since the magnitude of this third derivative is decreasing for $x \in (0,\pi),$ we can simply evaluate the magnitude of this derivative at $\frac{y}{2} - \frac{3 \pi}{4s}.$
   We thus have $\frac{d^3}{dx^3}\cot(x) =  \frac{-2(2+\cos(2x))}{\sin^4(x)},$ whose magnitude is at most $\frac{6}{(2x/\pi)^4}$ for $x \in (0,\pi].$ For $y \in [3\pi/s, \pi],$  we trivially have that $\frac{y}{2} - \frac{3 \pi}{4s} \ge \frac{y}{4},$ and thus we have the following bound: $$|\cot(\frac{y}{2}-\frac{3 \pi}{4 s}) - 3\cot(\frac{y}{2}-\frac{ \pi}{4 s})+3\cot(\frac{y}{2}+\frac{ \pi}{4 s}) - \cot(\frac{y}{2}+\frac{3 \pi}{4 s})| \le \left(\frac{\pi}{2s}\right)^3 \frac{6}{(y/2\pi)^4} \le \frac{12 \pi^7}{y^4 s^3}.$$  Since $g_2(y)$ is a symmetric function, the same bound holds for $y \in [-\pi,-3 \pi/s]$. Thus $|g_2(y)| \le \frac{12 \pi^7}{16s\cdot y^4 s^3} < \frac{\pi^7}{y^4 s^4}$ for $y \in [-\pi,\pi] \setminus (-3\pi/s,3 \pi/s).$  To conclude, note that $g_2(y)$ attains a global maximum at $y=0$, with $g_2(0)=\frac{1}{16 s} \left(6 \cot(\pi/4s) - 2 \cot(3 \pi/4s)\right) \le \frac{1}{16 s} \frac{24 s}{\pi} < 1/2.$
 \end{proof}

We now prove the final bullet point of Proposition~\ref{lemma:chebyshevEMS_b}.
\begin{lemma}
The Chebyshev earthmoving scheme is $\left[O(1/\sqrt{w}),\frac{w}{n \log n}\right]$-good, for any $w \in [1, \log n],$ where the $O$ notation hides an absolute constant factor.
\end{lemma}
\begin{proof}
We split this proof into two parts:  first we will consider the cost of the portion of the scheme associated with all but the first $s+1$ bumps, and then we  consider the cost of the skinny bumps $f_i$ with $i \in \{0,\ldots,s\}.$

For the first part, we consider the cost of bumps $f_i$ for $i\geq s+1$; that is the relative earthmover cost of moving $poi(xn,i)$ mass from $x$ to $\frac{i}{n}$, summed over $i\geq s$.  By definition of relative earthmover distance, the cost of moving mass from $x$ to $\frac{i}{n}$ is $|\log\frac{xn}{i}|$, which, since $\log y\leq y-1$, we bound by $\frac{xn}{i}-1$ when $i<xn$ and $\frac{i}{xn}-1$ otherwise. We thus split the sum into two parts.

For $i\geq \lceil xn\rceil$ we have $poi(xn,i)(\frac{i}{xn}-1)=poi(xn,i-1)-poi(xn,i)$. This expression telescopes when summed over $i\geq\max\{s,\lceil xn\rceil\}$ to yield $poi(xn,\max\{s,\lceil xn\rceil\}-1)=O(\frac{1}{\sqrt{s}})$.

For $i\leq\lceil xn\rceil-1$ we have, since $i\geq s$, that $poi(xn,i)(\frac{xn}{i}-1)\leq poi(xn,i)((1+\frac{1}{s})\frac{xn}{i+1}-1)=(1+\frac{1}{s})poi(xn,i+1)-poi(xn,i)$. The $\frac{1}{s}$ term sums to at most $\frac{1}{s}$, and the rest telescopes to $poi(xn,\lceil xn\rceil)-poi(xn,s)=O(\frac{1}{\sqrt{s}})$.  Thus in total, $f_i$ for $i\geq s+1$ contributes $O(\frac{1}{\sqrt{s}})$ to the relative earthmover cost, per unit of weight moved.
\medskip

We now turn to the skinny bumps $f_i(x)$ for $i\le s$. The simplest case is when $x$ is outside the region that corresponds to the cosine of a real number --- that is, when $xn\geq 4s$. It is straightforward to show that $f_i(x)$ is very small in this region. We note the general expression for Chebyshev polynomials: $T_j(x)=\frac{1}{2}\left[(x-\sqrt{x^2-1})^j+(x+\sqrt{x^2-1})^j\right]$, whose magnitude we bound by $|2x|^j$. Further, since $2x\leq\frac{2}{e}e^x$, we bound this by $(\frac{2}{e})^je^{|x|j}$, which we apply when $|x|>1$. Recall the definition $f_i(x)=t_i(1-\frac{xn}{2s})\sum_{j=0}^{s-1}poi(xn,j)$, where $t_i$ is the polynomial defined so that $t_i(\cos(y))=g_3^i(y)$, that is, $t_i$ is a linear combination of Chebyshev polynomials of degree at most $s$ and with coefficients summing in magnitude to at most $2$, as was shown in the proof of Lemma~\ref{lemma:ChebasP}. Since $xn>s$, we may bound $\sum_{j=0}^{s-1}poi(xn,j)\leq s\cdot poi(xn,s)$. Further, since $z\leq e^{z-1}$ for all $z$, letting $z=\frac{x}{4s}$ yields $x\leq 4s\cdot e^{\frac{x}{4s}-1}$, from which we may bound $poi(xn,s)=\frac{(xn)^s e^{-xn}}{s!}\leq\frac{e^{-xn}}{s!}(4s\cdot e^{\frac{xn}{4s}-1})^s=\frac{4^s s^s}{e^s\cdot e^{3xn/4} s!}\leq 4^se^{-3xn/4}$. We combine this with the above bound on the magnitude of Chebyshev polynomials, $T_j(z)\leq(\frac{2}{e})^je^{|z|j}\leq(\frac{2}{e})^se^{|z|s}$, where $z=(1-\frac{xn}{2s})$ yields $T_j(z)\leq(\frac{2}{e^2})^se^{\frac{xn}{2}}$. Thus $f_i(x)\leq poly(s) 4^se^{-3xn/4}(\frac{2}{e^2})^se^{\frac{xn}{2}}=poly(s)(\frac{8}{e^2})^s e^{-\frac{xn}{4}}$. Since $\frac{xn}{4}\geq s$ in this case, $f_i$ is exponentially small in both $x$ and $s$; the total cost of this earthmoving scheme, per unit of mass above $\frac{4s}{n}$ is obtained by multiplying this by the logarithmic relative distance the mass has to move, and summing over the $s+1$ values of $i\le s$, and thus remains exponentially small, and is thus trivially bounded by $O(\frac{1}{\sqrt{s}})$.

To bound the cost in the remaining case, when $xn\leq 4s$ and $i\le s$, we work with the trigonometric functions $g_3^{i}$, instead of $t_{i}$ directly.  Since mass may be moved freely below probability $\frac{w}{n \log n},$ we may assume that all the mass below this value is located at probability exactly $\frac{w}{n \log n}$.

For $y\in(0,\pi]$, we seek to bound the per-unit-mass relative earthmover cost of, for each $i\ge 0$, moving $g_3^i(y)$ mass from $\frac{2s}{n}(1-\cos(y))$ to $c_i$.  By the above comments, it suffices to consider $y \in [O(\frac{\sqrt{w}}{s}),\pi].$  This contribution is at most  $$\sum_{i=0}^s |g_3^i(y)\left( \log (1-\cos(y)) - \log(1-\cos(\frac{i \pi}{s}))\right)|.$$  We analyze this expression by first showing that for any $x,x' \in (0,\pi],$ $$\left|\log(1-\cos(x))-\log(1-\cos(x'))\right| \le 2|\log x - \log x'|.$$

 Indeed, this holds because the derivative of $\log (1-cos(x))$ is positive, and strictly less than the derivative of $2\log x$; this can be seen by noting that the respective derivatives are $\frac{\sin(y)}{1-\cos(y)}$ and $\frac{2}{y}$, and we claim that the second expression is always greater. To compare the two expressions, cross-multiply and take the difference, to yield $y\sin y-2+2\cos y$, which we show is always at most 0 by noting that it is 0 when $y=0$ and has derivative $y\cos y-\sin y$, which is negative since $y < \tan y$. Thus we have that $|\log (1-\cos(y))-\log(1-\cos(\frac{i \pi}{s}))|\leq 2|\log y-\log \frac{ i \pi}{s}|$; we use this bound in all but the last step of the analysis.  Additionally, we ignore the $\sum_{j=0}^{s-1}poi(xn,j)$ term as it is always at most 1.

We will now show that $$|g_3^0(y)(\log y-\log\frac{\pi}{s})|+\sum_{i=1}^{s} |g_3^i(y)(\log y-\log\frac{i\pi}{s})| = O(\frac{1}{sy}),$$ where the first term is the contribution from $f_0,c_0$.  For $i$ such that  $y\in(\frac{(i-3)\pi}{s},\frac{(i+3)\pi}{s})$, by the second bounds on $|g_2|$ in the statement of Lemma~\ref{lemma:g3bounds}, $g_3^i(y)<1$, and for each of the at most 6 such $i$, $|(\log y-\log\frac{\max\{1,i\}\pi}{s})| < \frac{1}{sy}$, to yield a contribution of $O(\frac{1}{sy})$.  For the contribution from $i$ such that $y \le \frac{(i-3)\pi}{s}$ or $y \ge \frac{(i-3)\pi}{s}$, the first bound of Lemma~\ref{lemma:g3bounds} yields $|g_3^i(y)|=O(\frac{1}{(ys- i\pi)^4})$.  Roughly, the bound will follow from noting that this sum of inverse fourth powers is dominated by the first few terms.   Formally, we split up our sum over $i \in [s] \setminus [\frac{ys}{\pi}-3, \frac{ys}{\pi}+3]$  into two parts according to whether $i>ys/\pi$:
\begin{eqnarray}
\sum_{i\ge\frac{ys}{\pi}+3}^s \frac{1}{(ys-i\pi)^4}|(\log y-\log\frac{i\pi}{s})| & \le & \sum_{i\ge\frac{ys}{\pi}+3}^\infty \frac{\pi^4}{(\frac{ys}{\pi}-i)^4}(\log i - \log \frac{ys}{\pi}) \nonumber \\
& \le & \pi^4 \int_{w=\frac{ys}{\pi} +2}^\infty \frac{1}{(\frac{ys}{\pi}-w)^4}(\log w - \log \frac{ys}{\pi}). \label{eq:sumc1}
\end{eqnarray}

Since the antiderivative of $\frac{1}{(\alpha-w)^4}(\log w - \log \alpha)$ with respect to $w$ is $$\frac{-2w(w^2-3w\alpha+3\alpha^2)\log w + 2(w-\alpha)^3\log(w-\alpha)+\alpha(2w^2-5w\alpha+3\alpha^2+2\alpha^2\log\alpha)}{6(w-\alpha)^3 \alpha^3},$$ the quantity in Equation~\ref{eq:sumc1} is equal to the above expression evaluated with $\alpha=\frac{ys}{\pi}$, and $w=\alpha+2$, to yield $$O(\frac{1}{ys})-\log \frac{ys}{\pi} + \log (2+\frac{ys}{\pi}) = O(\frac{1}{ys}).$$

A nearly identical argument applies to the portion of the sum for $i\le \frac{ys}{\pi}+3,$ yielding the same asymptotic bound of $O(\frac{1}{ys}).$   As it suffices to consider $y \ge O(\frac{\sqrt{w}}{s}),$ this bounds the total per-unit mass $\frac{w}{n \log n}$-truncated relative earthmover cost as $O(\frac{1}{\sqrt{w}}),$ as desired.

\end{proof}

\subsection{Proof of Theorem~\ref{thm:h2}}\label{sec:fp2}

We now assemble the key propositions from the above sections to complete our proof of Theorem~\ref{thm:h2}.

Proposition~\ref{lemma:faithfulwhp_b} guarantees that with high probability, the samples  will be ``faithful''.   For the remainder of the proof, we will assume that we are working with a faithful set of $n$ independent draws from a distribution with true histogram $h$.   Proposition~\ref{prop:goodfeasible} guarantees that there exists a feasible point $(v_x)$ for the linear program of Algorithm~\ref{def:alg1} with objective function at most $O(n^{\frac{1}{2}+\constB + \constC})$, such that if the empirical fingerprint above probability $\frac{n^\constB+2n^\constC}{n}$ is appended, the resulting histogram $h_1$ satisfies $R_\tau(h,h_1) \le O\left(\max(n^{-\constB(\frac{1}{2}-\constD)},n^{-(\constB-\constC)})\right),$ for any $\tau \ge 1/n^{3/2}.$

Let $h_2$ denote the histogram resulting from Algorithm~\ref{def:alg1}.   Hence the portion of $h_2$ below probability $\frac{n^{\constB}+2n^{\constC}}{n}$ corresponds to a feasible point of the linear program with  objective function bounded by $O(n^{\frac{1}{2}+\constB+\constC})$.  Additionally, $h_1(x)$ and $h_2(x)$ are identical for all $x > \frac{n^{\constB}+n^{\constC}}{n},$ as, by construction, they are both zero for all $x \in (\frac{n^{\constB}+n^{\constC}}{n},\frac{n^{\constB}+2n^{\constC}}{n}]$, and are both equal to the empirical distribution of the samples above this region.    We will now leverage the Chebyshev earthmoving scheme, via Proposition~\ref{lemma:chebyshevEMS_b} to argue that for any $w \in [1,\log n],$ $R_{\frac{w}{n\log n}}(h_1,h_2) \le O(\frac{1}{\sqrt{w}}),$ and hence by the triangle inequality, $R_{\frac{w}{n\log n}}(h,h_2) \le O(\frac{1}{\sqrt{w}}).$

To leverage the Chebyshev earthmoving scheme, recall that the earthmoving scheme that moves all the probability mass of a histogram to a discrete set of ``bump centers'' $(c_i)$, such that the earth moving scheme incurs a small truncated relative earthmover distance, and also has the property that when applied to any histogram $g$, the amount of probability mass that ends up at each bump center, $c_i$ is given as $\sum_{j\ge 0}\alpha_{i,j} \sum_{x:g(x) \neq 0} \poi(nx,j)x g(x),$ for some set of coefficients $\alpha_{i,j}$ satisfying for all $i,$  $\sum_{j \ge 0}|\alpha_{i,j}| \le 2 n^{0.3}.$

Consider the results of applying the Chebyshev earthmoving scheme to histograms $h_1$ and $h_2$.  We first argue that the discrepancy in the amount of probability mass that results at the $i$th bump center will be negligible for any $i \ge n^\constB + 2n^\constC.$  Indeed, since $h_1$ and $h_2$ are identical above probability $\frac{ n^\constB + n^\constC}{n}$ and $\sum_{i \ge n^\constB+2\constC} \poi(\lambda,i) = e^{-n^{\Omega(1)}}$ for $\lambda \le n^{\constB} +n^\constC,$ the discrepancy in the mass at all bump centers $c_i$ for $i \ge \ge n^\constB + 2n^\constC$ is trivially bounded by $o(1/n).$

We now address the discrepancy in the mass at the bump centers $c_i$ for $i<n^\constB + 2n^\constC.$  For any such $i$ the discrepancy is bounded by the following quantity:
\begin{eqnarray*}
\left| \sum_{j\ge 0}\alpha_{i,j} \sum_{x:h(x) \neq 0} \poi(nx,j)x \left(h_1(x) - h_2(x)\right) \right|& = &\left| \sum_{j\ge 0}\sum_{x:h(x) \neq 0} \alpha_{i,j} \frac{j+1}{n} \poi(nx,j+1) \left(h_1(x) - h_2(x)\right)\right| \nonumber \\
 & \le & \sum_{j\ge 1}\alpha_{i,j-1} \frac{j}{n} \left| \sum_{x:h(x) \neq 0}  \poi(nx,j) \left(h_1(x) - h_2(x)\right)\right| \nonumber \\
 & \le & o(1/n) + \sum_{j = 1}^{n^{\constB}+4n^{\constC}} \alpha_{i,j-1} \frac{j}{n} \left| \sum_{x:h(x) \neq 0}  \poi(nx,j) \left(h_1(x) - h_2(x)\right)\right| \nonumber \\
 & \le &  n^{0.3} \frac{(n^{\constB}+4n^{\constC})^2}{n}\cdot O(n^{\frac{1}{2}+\constB+\constC}) \nonumber \\
 & = & O(n^{0.3+\frac{1}{2}+3 \constB+\constC - 1}).
\end{eqnarray*}
Where, in the third line, we leveraged the bound $\sum_{j} |\alpha_{i,j}| \le n^{0.3}$ and the bound of $O(n^{\frac{1}{2}+\constB+\constC})$ on the linear program objective function corresponding to $h_1$ and $h_2$, which measures the discrepancies between $\sum_x \poi(nx,j) h_{\cdot}(x)$ and the corresponding fingerprint entries.   Note that the entirety of this discrepancy can be trivially equalized at a relative earthmover cost of $$O(n^{0.3+\frac{1}{2}+3 \constB+\constC - 1}\log(n)),$$ by, for example, moving this discrepancy to probability value $1$.   To complete the proof, by the triangle inequality we have that for any $w \in [1,\log n],$ letting $g_1$ and $g_2$ denote the respective results of applying the Chebyshev earthmoving scheme to histograms $h_1$ and $h_2$, we have the following:
\begin{eqnarray*}
R_{\frac{w}{n \log n}}(h,h_2)&  \le & R_{\frac{w}{n \log n}}(h,h_1)+ R_{\frac{w}{n \log n}}(h_1, g_1)+R_{\frac{w}{n \log n}}(g_1,g_2)+R_{\frac{w}{n \log n}}(g_2, h_2) \nonumber \\
& \le & O\left(\max(n^{-\constB(\frac{1}{2}-\constD)},n^{-(\constB-\constC)})\right) + O(1/\sqrt{w}) + O(n^{0.3+\frac{1}{2}+3 \constB+\constC - 1}\log(n)) + O(1/\sqrt{w}) \nonumber \\
& \le & O(1/\sqrt{w}).
\end{eqnarray*}

\section{Rounding a Generalized Histogram}\label{sec:round}

Algorithm~\ref{def:alg1} returns a generalized histogram. Recall that generalized histograms are histograms but without the condition that their values are integers, and thus may not correspond to actual distributions---whose histogram entries are always integral.  While a generalized distribution suffices to establish Theorem~\ref{thm:main}, we observe that it is possible to round a generalized histogram without significantly altering it, in truncated relative earthmover distance.  The following algorithm and lemma characterizing its performance show one way to round the generalized histogram to obtain a histogram that is close in truncated relative earthmover distance.  This, together with Theorem~\ref{thm:h2}, establishes Corollary~\ref{prop:hist}.

 \begin{center}
\myalg{alg:rounding}{Round to Histogram}{
\textbf{Input:} Generalized histogram $g$.\\
\textbf{Output:} Histogram $h.$
\begin{itemize}
       \item{Initialize $h$ to consist of the integral elements of $g$.}
       \item{For each integer $j\geq 0$:
           \begin{itemize}
           \item{Let $x_{j1},x_{j2},\ldots,x_{j\ell}$ be the elements of the support of $g$ that lie in the range $[2^{-(j+1)},2^{-j}]$ and that have non-integral histogram entries; let $m\defeq\sum_{i=1}^\ell x_{ji}g(x_{ji})$ be the total mass represented; initialize histogram $h'$ to be identically 0 and set variable $diff\defeq 0$.}
           \item{For $i=1,\ldots,\ell$:
               \begin{itemize}
               \item{ If $diff \leq 0$ set $h'(x_{ji}) = \lceil g(x_{ji}) \rceil$, otherwise, if $diff > 0$ set $h'(x_{ji}) = \lfloor g(x_{ji}) \rfloor.$}
               \item{Increment $diff$ by $x_{ji}\left(h'(x_{ji})-g(x_{ji})\right).$}
               \end{itemize}}
           \item{For each $i\in {1,\ldots,\ell}$ increment $h(\frac{m}{m+diff}\,x_{ji})$ by $h'(x_{ji})$.}
           \end{itemize}}
\end{itemize}}
\end{center}

\begin{lemma}\label{lemma:rounding}
Let $h$ be the output of running Algorithm~\ref{alg:rounding} on generalized histogram $g.$  The following conditions hold:
\begin{itemize}
\item{For all $x$, $h(x) \in \mathbb{N} \cup\{0\},$ and $\sum_{x:h(x)\neq 0} xh(x) =1,$ hence $h$ is a histogram of a distribution.}
\item{$R_0(h,g) \le 20\alpha$ where $\alpha \defeq\max(x:g(x) \not \in \mathbb{N} \cup\{0\}).$}
\end{itemize}
\end{lemma}
\begin{proof}
For each stage $j$ of Algorithm~\ref{alg:rounding}, the algorithm goes through each of the histogram entries $g(x_{ji})$ rounding them up or down to corresponding values $h'(x_{ji})$ and storing the cumulative difference in probability mass in the variable $diff$. Thus if this region of $g$ initially had probability mass $m$, then $h'$ will have probability mass $m+diff$. We bound this by noting that since the first element of each stage is always rounded up, and $2^{-(j+1)}$ is the smallest possible coordinate in this stage, the mass of $h'$, namely $m+diff$, is thus always at least $2^{-(j+1)}$. Since each element of $h'$ is scaled by $\frac{m}{m+diff}$ before being added to $h$, the total mass contributed by stage $j$ to $h$ is exactly $m$, meaning that each stage of rounding is ``mass-preserving".

Denoting by $g_j$ the portion of $g$ considered in stage $j$, and denoting by $h_j$ this stage's contribution to $h$, we now seek to bound $R(h_j,g_j)$.

Recall the \emph{cumulative distribution}, which for any distribution over the reals, and any number $y$, is the total amount of probability mass in the distribution between 0 and $y$. Given a generalized histogram $g$, we can define its (generalized) cumulative distribution by $c(g)(x)\defeq \sum_{x\leq y:g(x)\neq 0} xg(x)$. We note that at each stage $j$ of Algorithm~\ref{alg:rounding} and in each iteration $i$ of the inner loop, the variable $diff$ equals the difference between the cumulative distributions of $h'$ and $g_j$ at $x_{ji}$, and hence also on the region immediately to the right of $x_{ji}$. Further, we note that at iteration $i$, $|diff|$ is bounded by $x_{ji}$ since at each iteration, if $diff$ is positive it will decrease and if it is negative it will increase, and since $h'(x_{ji})$ is a rounded version of $g(x_{ji})$, $diff$ will be changed by $x_{ji}(h'(x_{ji})-g(x_{ji}))$ which has magnitude at most $x_{ji}$. Combining these two observations yields that for all $x$, $|c(h')(x)-c(g_j)(x)|\leq x$.

To bound the relative earthmover distance we note that for distributions over the reals, the earthmover distance between two distributions can be expressed as the integral of the absolute value of the difference between their cumulative distributions; since \emph{relative} earthmover distance can be related to the standard earthmover distance by changing each $x$ value to $\log x$, the change of variables theorem gives us that $R(a,b)=\int \frac{1}{x} |c(b)(x)-c(a)(x)|\,dx$. We can thus use the bound from the previous paragraph in this equation after one modification: since $h'$ has total probability mass $m+diff$, its relative earthmover distance to $g_j$ with probability mass $m$ is undefined, and we thus define $h''$ to be $h'$ with the modification that we subtract $diff$ probability mass from location $2^{-j}$ (it does not matter to this formalism if $diff$ is negative, or if this makes $h''(2^{-j})$ negative). We thus have that $R(h'',g_j)=\int_{2^{-(j+1)}}^{2^{-j}} \frac{1}{x} |c(h')(x)-c(g_j)(x)|\,dx\leq \int_{2^{-(j+1)}}^{2^{-j}} \frac{1}{x} x\,dx=2^{-(j+1)}$.

We now bound the relative earthmover distance from $h''$ to $h_j$ via the following two-part earthmoving scheme: all of the mass in $h''$ that comes from $h'$ (specifically, all the mass except the $-diff$ mass added at $2^{-j}$) is moved to a $\frac{m}{m+diff}$ fraction of its original location, at a relative earthmover cost $(m+diff)\cdot|\log\frac{m}{m+diff}|$; the remaining $-diff$ mass is moved wherever needed, involving changing its location by a factor as much as $2\cdot\max\{\frac{m}{m+diff},\frac{m+diff}{m}\}$ at a relative earthmover cost of at most $|diff|\cdot(\log 2+|\log\frac{m}{m+diff}|)$. Thus our total bound on $R(g_j,h_j)$, by the triangle inequality, is $2^{-(j+1)}+(m+diff)\cdot|\log\frac{m}{m+diff}|+|diff|\cdot(\log 2+|\log\frac{m}{m+diff}|)$, which we use when $m\geq 2^{-j}$, in conjunction with the two bounds derived above, that $|diff|\leq 2^{-j}$ and that $m+diff\geq 2^{-(j+1)}$, yielding a total bound on the earthmover distance of $5\cdot 2^{-j}$ for the $j$th stage when $m\geq 2^{-j}$. When $m\leq 2^{-j}$ we note directly that $m$ mass is being moved a relative distance of at most $2\cdot\max\{\frac{m}{m+diff},\frac{m+diff}{m}\}$ at a cost of $m\cdot(\log 2+|\log\frac{m}{m+diff}|)$ which we again bound by $5\cdot 2^{-j}$.  Thus, summing over all $j\geq \lfloor|\log_2 \alpha|\rfloor$, yields a bound of $20\alpha$.
\end{proof}

\end{document}